\documentclass[sigconf,authorversion]{acmart}

\usepackage{graphicx}

\usepackage{bm}
\usepackage{adjustbox}
\usepackage{balance}
\usepackage{amsmath}
\usepackage{amsfonts}
\usepackage{amsthm}
\usepackage{dsfont}
\usepackage{subcaption}
\usepackage{cleveref}
\usepackage[ruled,vlined,linesnumbered]{algorithm2e}
\usepackage{enumitem}
\usepackage{tikz}

\newcommand\numberthis{\addtocounter{equation}{1}\tag{\theequation}}

\newcommand{\algname}{\textsc{SamRuLe}}

\newcommand{\ruleset}{\mathcal{R}}

\newcommand{\pars}[1]{\left( #1 \right)}
\newcommand{\sqpars}[1]{\left[ #1 \right]}
\newcommand{\brpars}[1]{\left\lbrace #1 \right\rbrace}
\newcommand{\qt}[1]{\lq\lq#1\rq\rq}
\newcommand{\qtm}[1]{\text{\lq\lq}#1\text{\rq\rq}}
\newcommand{\ind}[1]{\mathds{1} \left[ #1 \right] }

\newcommand{\defsymbol}{\intercal}
\newcommand{\defrule}{r_{\defsymbol}}

\newcommand{\dataset}{\mathcal{D}}
\newcommand{\sample}{\mathcal{S}}

\newcommand{\BLi}[1]{\Omega ( #1 )}

\newcommand{\BOi}[1]{\mathcal{O}( #1 )}

\newcommand{\E}{\mathop{\mathbb{E}}}

\newcommand{\mC}{\mathcal{C}}
\newcommand{\mG}{\mathcal{G}}

\DeclareMathOperator*{\argmin}{arg\,min}

\newcommand{\logtdelta}{\ln \bigl( \frac{2}{\delta} \bigr) }

\newif\ifextversion
\extversiontrue

\AtBeginDocument{%
  }

\copyrightyear{2024} 
\acmYear{2024} 
\setcopyright{rightsretained} 
\acmConference[KDD '24]{Proceedings of the 30th ACM SIGKDD Conference on Knowledge Discovery and Data Mining}{August 25--29, 2024}{Barcelona, Spain}
\acmBooktitle{Proceedings of the 30th ACM SIGKDD Conference on Knowledge Discovery and Data Mining (KDD '24), August 25--29, 2024, Barcelona, Spain}
\acmDOI{10.1145/3637528.3671989}
\acmISBN{979-8-4007-0490-1/24/08}

\begin{document}

\title{Scalable Rule Lists Learning with Sampling}

\author{Leonardo Pellegrina}
\affiliation{%
  \institution{Dept. of Information Engineering, University of Padova}
  \city{Padova}
  \country{Italy}
}
\email{leonardo.pellegrina@unipd.it}

\author{Fabio Vandin}
\affiliation{%
  \institution{Dept. of Information Engineering, University of Padova}
  \city{Padova}
  \country{Italy}
}
\email{fabio.vandin@unipd.it}

\renewcommand{\shortauthors}{Leonardo Pellegrina and Fabio Vandin}

\begin{abstract}
Learning interpretable models has become a major focus of machine learning research, given the increasing prominence of machine learning in socially important decision-making. Among interpretable models, rule lists are among the best-known and easily interpretable ones. However, finding optimal rule lists is computationally challenging, and current approaches are impractical for large datasets.

We present a novel and scalable approach to learn nearly optimal rule lists from large datasets. Our algorithm uses sampling to efficiently obtain an approximation of the optimal rule list with rigorous guarantees on the quality of the approximation. In particular, our algorithm guarantees to find a rule list with accuracy very close to the optimal rule list when a rule list with high accuracy exists. Our algorithm builds on the VC-dimension of rule lists, for which we prove novel upper and lower bounds. Our experimental evaluation on large datasets shows that our algorithm identifies nearly optimal rule lists with a speed-up up to two orders of magnitude over state-of-the-art exact approaches. Moreover, our algorithm is as fast as, and sometimes faster than, recent heuristic approaches, while reporting higher quality rule lists. In addition, the rules reported by our algorithm are more similar to the rules in the optimal rule list than the rules from heuristic approaches.
 \end{abstract}




\maketitle


\section{Introduction}
\label{sec:intro}

Interpretability is one of the characteristics of machine learning models that has become a major topic of research due to the ever increasing impact of machine learning models in socially important decision-making~\cite{interpretablemlsurvery}. The goal is to build models that are easily understood by humans, while achieving high predictive power, in contrast to black-box models (e.g., large deep-learning models) that are highly predictive but are not transparent nor interpretable.

Rule lists~\cite{rivest1987learning}, and more in general rule-based models such as decision trees, are among the best-known and easily interpretable models. A rule list is a sequence of rules, and the prediction from such a model is obtained  by applying the first rule in the list whose condition is satisfied for the given input.

Rule-based models often display predictive power comparable to black-box models in several critical applications~\cite{interpretablemlsurvery}. While advanced techniques have been recently developed to speed-up the discovery of optimal rule lists~\cite{angelino2018learning,rudin2018learning,okajima2019decision,yu2021learning}, it remains a computationally challenging task, in particular for large datasets where differentiable models (such as neural networks) can be easily learned with gradient descent and its variants.

Fast heuristic methods for rule lists and other rule-based models have been developed~\cite{cohen1995fast,yang2017scalable}, but they do not provide guarantees in terms of the accuracy of the rule list they report. In several applications identifying an optimal, or close to optimal rule, is crucial, since the goal is to find models that are not only interpretable, but also have high predictive power and can therefore be used as interpretable alternatives to black-box models.

A natural solution to speed up the training on a large dataset, and to obtain an approximation of the best rule list, is to reduce the size of the training set, e.g., by only considering a small random sample of the dataset. 
On the other hand, it is clear that there is an intrinsic trade-off between the accuracy of this approximation and the size of the random sample, i.e., the computational cost incurred by the learning algorithm.  
This accuracy depends directly on the fluctuations of the losses estimated on the sample w.r.t. the losses measured on the entire dataset: 
a sample too small can be analyzed quickly but may provide inaccurate estimates.  
Moreover, it is possible that the best rule list, learned from a random sample, may be qualitatively different, thus not representative, of the optimal solution computed on the whole dataset. 

\textit{Contributions.} 
To address these challenges, we precisely quantify the maximum deviation between the accuracies of rule lists that are estimated on the random sample w.r.t. the respective exact counterparts.  As a consequence, we obtain sufficient sample sizes that guarantee that an accurate rule list can be found from a small random sample. 
We then propose \algname\ (\textit{Sam}pling for \textit{Ru}les list \textit{Le}arning), a scalable approach to approximate optimal rule lists from large datasets with rigorous guarantees on the quality of the approximation. 

\begin{itemize}[leftmargin=0.5cm]
\item \algname\ is the first scalable approach to approximate optimal rule lists from large datasets while providing rigorous guarantees on the accuracy of the reported rule list. In particular, \algname\ guarantees to find a rule list with accuracy very close to the optimal one when a rule list with high accuracy exists. As a result, \algname\ makes it possible to assess the existence of high accuracy rule lists for large datasets.
\item \algname\ uses sampling to efficiently obtain an approximation of the optimal rule list with guarantees. We derive novel bounds relating the accuracy of a rule lists in the sample with its accuracy in the whole dataset \emph{and} with the accuracy of the optimal rule list. Our bounds build on the VC-dimension of rule lists, for which we prove novel upper and lower bounds, and that allow us to define sample sizes leading to the desired quality guarantees.
\item Our experimental evaluation shows that \algname\ enables the identification of nearly optimal rule lists from large datasets, with a speed-up of up to two orders of magnitude over state-of-the-art exact approaches. Moreover, \algname\ is as fast as, or faster than, recent heuristic approaches, while reporting higher quality rule lists, for which it can also certify their quality with respect to the optimal rule list. The rules reported by \algname\ are also more similar to the rules in the optimal rule list than the rules from heuristic approaches.
\end{itemize}

\section{Related Works}
\label{sec:relworks}

We now discuss the works most related to ours. We focus on methods to learn rule lists, in particular \emph{sparse} rule lists. Sparsity is a crucial requirement for the interpretability of models. For an overview on the topic of interpretability in machine learning, we point the interested reader to the survey from Rudin et al.~\cite{interpretablemlsurvery}.

Finding  optimal sparse rule lists is NP-hard, and several methods have been proposed to learn sparse rule lists. Such methods fall under two categories: 
exact methods, that guarantee to find the optimal rule list (i.e., with highest accuracy), 
and heuristic methods, that are faster but provide no guarantees 
on the accuracy of the reported rule list w.r.t. the optimal. 

In the category of exact methods, Angelino et al.~\cite{angelino2018learning} propose CORELS, a branch-and-bound algorithm that achieves several orders of magnitude speedup in time and a significant reduction of memory consumption, by leveraging algorithmic bounds, efficient data structures, and computational reuse. Rudin and Ertekin~\cite{rudin2018learning} develop a mathematical programming approach to building rule lists. 
Okajima and Sadamasa~\cite{okajima2019decision} define a continuous relaxed version of  a rule list, and propose an algorithm that optimizes rule lists based on such continuous relaxed versions. 
Yu et al.~\cite{yu2021learning} propose a SAT-based approach to find optimal rule sets and lists. 
While highly optimized, exact methods such as the aforementioned ones only apply to problems of moderate size, and do not scale to large datasets. 

In the category of heuristic methods, Rivest~\cite{rivest1987learning} proposed greedy splitting techniques, also used in subsequent heuristic algorithms for learning decision trees~\cite{breiman84classification,quinlan1993c45}.  
RIPPER~\cite{cohen1995fast} builds rule sets in a greedy fashion, and a similar greedy strategy has been used for finding sets of robust rules in terms of minimum description length~\cite{fischer2019sets}.
Other works considered different problems and classification models, such as 
probabilistic rules for multiclass classification~\cite{proencca2020interpretable}, and, more recently, rule lists with preferred variables~\cite{papagianni2023discovering}, sets of locally optimal rules~\cite{huynh2023efficient},  
and multi-label rule sets~\cite{ciaperoni2023concise}. 
Yang et al.~\cite{yang2017scalable} uses, instead, a Bayesian approach with Monte-Carlo search to explore the rule list solution space. 
All such approaches focus on improving efficiency with respect to exact approaches, but do not provide guarantees on the quality of the reported rule lists with respect to the optimal one. 
In several challenging instances, the gap with the optimal solution may be substantial, resulting in a predictive model with suboptimal performance. 
 
A different line of research, related to the identification of almost-optimal models, is provided by the exploration of the \emph{Rashomon set}, that is the set of almost-optimal models for a machine learning problem~\cite{fisher2019all,breiman2001statistical,semenova2022existence}. 
Works in this area have focused on the exploration or full enumeration of the Rashomon set, including the case of rule-based models~\cite{xin2022exploring}, for example to study variable importance in well-performing models~\cite{dong2020exploring}. 
While some of our theoretical results may be useful for such applications as well, 
our focus is not the exploration or enumeration of \emph{all} almost optimal models, but rather on efficiently finding one nearly optimal model while providing guarantees on its quality.

\sloppy{
Our algorithm builds on novel upper bounds to the VC-dimension~\cite{Vapnik:1971aa} of rule lists, a fundamental combinatorial measure of the complexity of classification models. 
We combine these bounds with sharp concentration inequalities to bound the maximum deviation between the prediction accuracy of rule lists from a random sample w.r.t. the entire dataset. 
Our VC-dimension upper bounds generalize previous results that can be derived from rule-based classification models~\cite{rudin2013learning}, as they provide a more granular dependence on the rule list parameters compared to known results (see \Cref{sec:VCrulelists} for a detailed discussion). 
We establish the tightness of our analysis 
by proving almost matching lower bounds, that significantly improves over the best known results from Yildiz~\cite{yildiz2014vc}. 
The VC-dimension has been used to develop generalization bounds and sampling algorithms for other data mining tasks, including association rules mining~\cite{RiondatoU14} and sequential pattern mining~\cite{santoro2020mining}. 
To the best of our knowledge, our work is the first to use VC-dimension in designing efficient sampling approaches for learning rule lists.
}

\section{Preliminaries}
\label{sec:prelims}
We define a \emph{dataset} $\dataset$ as a collection $\dataset = \brpars{ (s_1 , t_1) , \dots , (s_n , t_n) }$ of $n$ training instances, where  
each training instance is a pair $(s , t) \in \dataset$ composed by a set $s$ of $d$ features, and a categorical label $t$. 
We denote with $x_{i}$ the $i$-th feature, for $1 \leq i \leq d$. 
For simplicity, in this work we will focus on datasets with binary features, $x_i \in \{ 0 , 1 \}$ for $1 \leq i \leq d$, and binary labels $t \in \{0 , 1\}$, but most of our results extend to the general case. We use $\dataset_s$ to denote the collection of the $d$ features of each training instance, ignoring the labels, that is $\dataset_{s} = \brpars{ s_1 ,s_2, \dots , s_n }$.

We define a rule list $R = [ r_{1} , r_{2} , \dots , r_{k} , \defrule ]$ of length $|R| = k$ as a sequence of $k \geq 0$ rules $r_{i} , 1 \leq i \leq k$, plus a \emph{default rule} $\defrule$. 
Each rule $r \in R$ is defined as a pair $r = (c , p)$ where $c$ is a condition on a subset of the $d$ features of the dataset, and $p \in \{0,1\}$ is a prediction of the label. 
Similarly to previous rule-based models~\cite{angelino2018learning,rudin2018learning,yang2017scalable}, we only consider conditions composed by conjunctions of monotone variables of the type $c = \bigwedge_{i} \qtm{x_{i}=1}$, i.e., we do not consider the negation of binary features (e.g., $\qtm{x_{i}=0}$). 
Note that negations can be included easily by considering additional binary features. 
The default rule $\defrule = ( c_{\defsymbol} , p )$ is composed by the condition $c_{\defsymbol}$ that is always true. 
Given an instance $(s,t) \in \dataset$, 
the predicted label $P(R,s)$ by the rule list $R$ on $s$ is computed as the predicted label $p$ of the first rule $r$ whose condition $c$ is satisfied by $s$; 
\Cref{fig:ruleexample} shows an example of a dataset with $n=5$ training instances over $d=4$ features, and a rule list $R$ of length $k=3$ with conjunctions with at most $z=2$ terms.

\begin{figure}[ht]
\begin{subfigure}{.23\textwidth}
\center
  \begin{tabular}{lrrrrrrr}
    \toprule
    $\dataset$              & $x_{1}$   & $x_{2}$ & $x_{3}$  & $x_{4}$ & $t$     \\
    \midrule
	$s_{1}$     & 0 & 1  & 0 & 0 & 1 \\
	$s_{2}$     & 1 & 1  & 0 & 0 & 1 \\
	$s_{3}$     & 0 &  0 & 1 & 1 & 1 \\
	$s_{4}$     & 0 &  0 & 0 & 0 & 0 \\
	$s_{5}$     & 1 &  0 & 1 & 1 & 0 \\
  \bottomrule
\end{tabular}
\caption{}
\end{subfigure}
\begin{subfigure}{.23\textwidth}
\tikzset{every picture/.style={line width=0.75pt}}
\begin{tikzpicture}[x=0.75pt,y=0.75pt,yscale=-1,xscale=1]
\draw (100,125) node [anchor=north west][inner sep=0.75pt]   [align=left] {{ \textbf{if} $x_{2} = 1\rightarrow 1$}};
\draw (100,140.5) node [anchor=north west][inner sep=0.75pt]   [align=left] {{ \textbf{else if} $x_{1} = 1 \wedge x_{3} = 1 \rightarrow 0$}};
\draw (100,155) node [anchor=north west][inner sep=0.75pt]   [align=left] {{ \textbf{else if} $x_{4} =1 \rightarrow 1$}};
\draw (100,170) node [anchor=north west][inner sep=0.75pt]   [align=left] {{ \textbf{else} $\rightarrow 0$}};
\end{tikzpicture}
\caption{}
\end{subfigure}
\caption{(a): example of a dataset $\dataset$ with $n=5$ instances and $d=4$ features. 
(b): example of a rule list $R$ with length $k=3$ with conjunctions with at most $z=2$ terms. 
The rule list $R$ perfectly classifies the instances of $\dataset$.}
\label{fig:ruleexample}
\Description{The figure shows an example of a dataset with 5 instances and 4 features, and an example of a rule list R with length 3 with conjunctions with at most 2 terms. The rule list R perfectly classifies the instances of the example dataset.}
\end{figure}

The \emph{regularized binary loss} $\ell(R,\dataset)$ of a rule list $R$ is defined, for some constant $\alpha \geq 0$, as
\begin{align*}
\ell(R,\dataset) = \frac{1}{n} \sum_{(s,t) \in \dataset} \ind{ P(R , s) \neq t } + \alpha |R|.
\end{align*}

The use of the \emph{regularized} loss is common practice in rule list learning~\cite{angelino2018learning}, allowing to prefer shorter (i.e., simpler) rule lists over longer ones when their accuracy is similar, as quantified by the parameter  $\alpha$.

Define $\ruleset_{k}^{z}$ as the set of rule lists with length $\leq k$, where each rule in $R$ is given by the conjunction of at most $z$ terms. 
The rule list learning problem is to identify a rule $R^{\star}$ such that 
\begin{align*}
R^{\star} = \argmin_{R \in \ruleset_{k}^{z}} \brpars{ \ell(R,\dataset) } .
\end{align*}
The values of $k$ and $z$ are usually fairly small, since large values would compromise the interpretability of the rule list~\cite{interpretablemlsurvery}. Finding $R^{\star}$ is NP-hard, and to develop efficient algorithms one has to resort to approximations. 

Our goal is to compute, for given accuracy parameters $\varepsilon , \theta \in (0,1]$, a rule list $\tilde{R}$ that provides an $(\varepsilon , \theta)$-approximation of the optimal rule list $R^{\star}$, defined as follows.
\begin{definition}
\label{def:guarantees}
A rule list $\tilde{R}$ provides an $(\varepsilon , \theta)$-approximation of the optimal rule list $R^{\star}$ if it holds
\begin{align*}
\ell(\tilde{R},\dataset) \leq \ell(R^{\star},\dataset) + \varepsilon \max\{ \ell(R^{\star},\dataset) , \theta \} .
\end{align*}
\end{definition}
Our definition of $(\varepsilon , \theta)$-approximation is motivated by the fact that the optimal loss $\ell(R^{\star},\dataset)$ is unknown a priori, and often extremely expensive to compute or even estimate from large datasets.  
Therefore, we need to design approximation guarantees that are sharp in all situations, i.e., for the all possible values of the optimal loss.  
The $(\varepsilon , \theta)$-approximation of \Cref{def:guarantees} allows to interpolate between an additive approximation, with small absolute error $\varepsilon \theta$, when the optimal loss $\ell(R^{\star},\dataset)$ is small (i.e., $\leq \theta$), and a relative approximation (with parameter $\varepsilon$) when the optimal loss $\ell(R^{\star},\dataset)$ is large (i.e., $> \theta$). 
This flexible design avoids statistical bottlenecks that are incurred when the loss of the best model is large, a well known issue in statistical learning theory~\cite{boucheron2005theory}.
Taking this effect into account allows \algname\ to be extremely efficient in terms of sample sizes. 

As we will formalize in \Cref{sec:algo}, the main goal of \algname\ is to compute an $(\varepsilon , \theta)$-approximation of the optimal rule list from a randomly drawn samples with high probability.

\section{\algname\ Algorithm}
\label{sec:algo}

As introduced in previous sections, the main idea of our approach is to compute an approximation of the best rule list performing the analysis of a small random sample, instead of processing a large dataset, which may be extremely expensive and unfeasible in practice. In this section we formally define our algorithm. 

Define a sample $\sample$ as a collection of $m$ instances of the dataset $\dataset$ taken uniformly and independently at random. 
The regularized loss of a rule list $R$ computed on the sample $\sample$ is defined as 
\begin{align*}
\ell(R,\sample) = \frac{1}{m} \sum_{(s,t) \in \sample} \ind{ P(R , s) \neq t } + \alpha |R|.
\end{align*}
Intuitively, $\ell(R,\sample)$ provides an \emph{estimate} of $\ell(R,\dataset)$. 
More precisely, the loss of any rule $R$ computed on $\sample$ is an \emph{unbiased} estimator of the loss computed on $\dataset$; in fact, it holds
\begin{align*}
\E_{\sample}[\ell(R,\sample)] = \ell(R,\dataset) . 
\end{align*}
We remark, however, that the convergence of $\ell(R,\sample)$ towards $\ell(R,\dataset)$ \emph{in expectation} is not sufficient to derive rigorous approximations for the best rule list; instead, it is necessary to take into account the variance of the deviations, that indirectly depends the \emph{complexity} of the class of rule list models. 
Therefore, our approach is based on properly bounding the \emph{deviations} $|\ell(R,\sample) - \ell(R,\dataset)|$ \emph{with high probability}, in order to guarantee that the best rule identified from the sample $\sample$ is a high-quality approximation of the optimal solution from $\dataset$. 
To obtain a scalable method, we wish to identify the smallest possible sample size that yields the above-mentioned guarantees. 
We address this challenging problem combining advanced VC-dimension bounds and sharp concentration inequalities. 

We now describe our algorithm \algname. The input to \algname\ is: a dataset $\dataset$ of $m$ training instances over $d$ features; the values $k$ and $z$, representing the maximum number of rules in a rule list and the maximum number of conditions for each rule; constants $\varepsilon, \theta  \in (0,1]$, that are the values defining the quality of the $(\varepsilon, \theta)$-approximation (see Section~\ref{sec:prelims}); and a constant $\delta \in (0,1]$, that defines the required \emph{confidence}, that is the probability that the output by \algname\ is a  $(\varepsilon, \theta)$-approximation. (Note that the set $\ruleset_{k}^{z}$ of rules lists that can be produced in output by \algname\ is defined by $\dataset$, $k$, and $z$.)  Given the input, \algname\ computes the quantity $\hat{m}(\ruleset_{k}^{z} , \dataset)$, defined as the minimum value of $m \geq 1$ such that
it holds 
\begin{align*}
\sqrt{ \frac{ 3\theta\ln(\frac{2}{\delta})}{m} } 
+  \sqrt{ \frac{2 \Bigl( \theta \! + \! \sqrt{ \frac{ 3\theta\ln(\frac{2}{\delta})}{m} }  \Bigr) \bigl( \omega \! + \! \ln(\frac{2}{\delta}) \bigr) }{m} } 
+ \frac{2 \bigl( \omega \! + \! \ln(\frac{2}{\delta}) \bigr) }{m} 
\leq \varepsilon \theta , 
\end{align*}
where $\omega = kz\ln(2ed/z)+2$. 
Note that $\hat{m}(\ruleset_{k}^{z} , \dataset)$ can be easily computed  (e.g., with a binary search over $m$). We also show analytical bounds to $\hat{m}(\ruleset_{k}^{z} , \dataset)$ in our analysis (Section~\ref{sec:guarantees}).

\algname\ than draws a sample $\sample$ of $\hat{m}(\ruleset_{k}^{z} , \dataset)$  instances taken uniformly at random from $\dataset$, and outputs $\tilde{R} = \argmin_{R \in \ruleset_{k}^{z}} \ell(R,\sample)$ by finding the optimal rule list on the sample $\sample$. Note that \algname\ can leverage any exact algorithm, such as CORELS~\cite{angelino2018learning}, to find $\tilde{R}$ from $\sample$.

In the following sections, we prove that  \algname\ provides an $(\varepsilon, \theta)$-approximation of the optimal rule list $R^\star = \argmin_{R \in \ruleset_{k}^{z}} \ell(R,\sample)$ with probability $\geq 1-\delta$. 
Our proof leverage novel bounds on the VC-dimension of rule lists, that we present in Section~\ref{sec:VCrulelists}, while the approximation guarantees from \algname\ are proved in Section~\ref{sec:guarantees}.

\subsection{Complexity of Rule Lists}
\label{sec:VCrulelists}

In this section we provide analytical results on the complexity of rule lists, by studying their VC-dimension~\cite{Vapnik:1971aa}.
Intuitively, the VC-dimension measures the capacity of a class of predicion models of predicting arbitrary labels assignments with perfect accuracy. 
As discussed in previous sections, these bounds will be instrumental to derive accurate approximation bounds and prove that the reported rule by \algname\ is an $(\varepsilon , \theta)$-approximation. 

We start by providing the main concepts related to VC-dimension (see~\cite{mitzenmacher2017probability,mohri2018foundations,ShalevSBD14} for a more in depth introduction). 

Given a dataset $\dataset = \brpars{ (s_1 , t_1) , \dots , (s_n , t_n) }$, remember that $\dataset_s = \brpars{ s_1 , \dots , s_n }$. Given the set $\ruleset_{k}^{z}$ of rule lists with length $\leq k$, where each condition contains conjunctions with at most $z$ terms, for each rule list $R \in \ruleset_{k}^{z}$ we define the \emph{projection} of $R$ on the dataset $\dataset$ as  $X(R , \dataset) = \{s \in \dataset_s: P(R,s) = 1\}$. 
We define the \emph{range space} associated to the dataset $\dataset$ and  to $\ruleset_{k}^{z}$ as $(\dataset, \tilde{\ruleset}_{k}^{z})$, where $\tilde{\ruleset}_{k}^{z} = \{ X(R , \dataset) : R \in \ruleset_{k}^{z} \}$ is the \emph{range set}. 
Note that in general $|\tilde{\ruleset}_{k}^{z}|$ can be smaller than $|\ruleset_{k}^{z}|$, since two rule lists $R_1$ and $R_2$ may provide the same predictions for all instances in a dataset $\dataset$, which implies $X(R_1 , \dataset) = X(R_2 , \dataset)$.  
The \emph{projection} of the range set $\tilde{\ruleset}_{k}^{z}$ on a subset $\dataset^{\prime} \subseteq \mathcal{D}$ of the dataset $\dataset$ is defined as $\mathcal{P}({\tilde{\ruleset}_{k}^{z}},\dataset^{\prime}) = \{X(R , \dataset^{\prime}) :  R \in \ruleset_{k}^{z} \}$. 

If $\mathcal{P}(\tilde{\ruleset}_k^z,\dataset^{\prime}) = 2^{\dataset^{\prime}}$, i.e., the projection of $\tilde{\ruleset}_k^z$ on $\dataset^{\prime}$ produces all possible dichotomies of the elements of $\dataset^{\prime}$, then $\dataset^{\prime}$ is \emph{shattered} by $\tilde{\ruleset}_{k}^{z}$. 
The VC-dimension of the range space $(\dataset, \tilde{\ruleset}_k^z)$ is the cardinality of the largest subset of $\dataset$ that is shattered by $\tilde{\ruleset}_{k}^{z}$. 
In what follows we will use $VC(\ruleset_{k}^{z})$ to denote the VC-dimension of the range space $(\dataset, \tilde{\ruleset}_{k}^{z})$ associated to the set of rule lists $\ruleset_{k}^{z}$.

\subsubsection{Upper bounds to the VC-dimension}
We now prove refined upper bounds to the VC-dimension of rule lists. 
We start by providing upper bounds to the \emph{growth function} (the maximum cardinality of distinct projections of a range set, see below) of rule lists $\ruleset_{k}^{1}$, for which $z=1$ (i.e., each rule consists of a condition on exactly 1 feature). 
From such bound we derive an upper bound to the VC-dimension of the range space associated to $\ruleset_{k}^{1}$, that we will generalize to derive an upper bound to the VC-dimension of $\ruleset_{k}^{z}$ for general values of $z$. 
Due to space constraints, some of the proofs are in Appendix.

Define the \emph{growth function} $\Lambda(\ruleset_{k}^{z},m)$ of the rule lists $\ruleset_{k}^{z}$ as the maximum number of \emph{distinct} projections of rule lists in $\ruleset_{k}^{z}$  over \emph{any} dataset $\dataset$ with $m$ instances, that is:
\begin{equation*}
\Lambda(\ruleset_{k}^{z},m) = \max_{\dataset^{\prime} \subseteq \dataset: |\dataset^{\prime}| = m} |\mathcal{P}(\tilde{\ruleset}_k^z,\dataset^{\prime})|. 
\end{equation*}

\begin{theorem}
\label{thm:gfunct}
It holds 
\begin{align*}
\Lambda(\ruleset_{k}^{1},m ) \leq 2 + \sum_{j=1}^{k} \sqpars{ 2^{j} \prod_{i=0}^{j-1}(d-i) } .
\end{align*}
\end{theorem}

\begin{proof}
For $0 \leq j \leq k$, 
we define the set $\mG_{j}$ of rule lists with length exactly $j$ as $\mG_{j} = \{ R \in \ruleset_{k}^{1} : |R| =j  \}$,
and define $\tilde{\mG}_{j} = \{ X(R , \dataset) : R \in \mG_{j} \}$
the range set for $\mG_{j}$,
with $\tilde{\ruleset}_k^1 = \cup_j \tilde{\mG}_{j}$.
From an union bound and Jensen's inequality, it holds
\begin{align*}
\Lambda(\ruleset_{k}^{1},m ) 
&= \max_{\dataset^{\prime} \subseteq \dataset: |\dataset^{\prime}| = m} |\mathcal{P}(\cup_j \tilde{\mG}_{j},\dataset^{\prime})| \\
&\leq \max_{\dataset^{\prime} \subseteq \dataset: |\dataset^{\prime}| = m} \sum_{j=0}^k |\mathcal{P}(\tilde{\mG}_{j},\dataset^{\prime})| 
\leq \sum_{j=0}^{k} \Lambda(\mG_{j},m ).
\end{align*}
We proceed to bound each $\Lambda(\mG_{j},m )$ separately, while excluding rule lists $R \in \mG_{j}$ that have the same projection, on any $\dataset^\prime$, of other rules $R^\prime \in \mG_{j^\prime}$, for some $j^\prime \neq j$.
First, consider $\mG_{0}$; 
it holds $\Lambda(\mG_{0},m ) = 2$ since $\mG_{0}$ contains $R_{1} = [(c_{\defsymbol},1)]$ and $R_{2} = [(c_{\defsymbol},0)]$.
Then, consider $\mG_{j}$ for any $1 \leq j \leq k$, and let any rule list $R \in \mG_{j}$, 
with $R = [(c_{1},p_{1}) , \dots , (c_{j},p_{j}) , (c_{\defsymbol},p)  ]$. 
There are two possibilities: $p_{j} = p$ or $p_{j} \neq p$. 
In the first case, denote the rule $R^{\prime} = [(c_{1},p_{1}) , \dots , (c_{j-1},p_{j-1}) , (c_{\defsymbol},p)  ] \in \mG_{j-1}$. 
It is easy to observe that the projections $X(R,\dataset^{\prime})$ and $X(R^{\prime},\dataset^{\prime})$ are equal, for any $\dataset^{\prime}$.
Therefore, any rule list $\in \mG_{j}$ with $p_{j} = p$ can be ignored from the upper bound to $\Lambda(\mG_{j},m )$ as it is already covered by at least one element of the set $\mG_{j-1}$.
Then, consider a rule list $R = [(c_{1},p_{1}) , \dots , (c_{j},p_{j}) , (c_{\defsymbol},p)  ] \in \mG_{j}$ such that there exist two indices $1 \leq y < i \leq j$ with $c_{y} = c_{i}$. 
We observe that the rule $R^{\prime} = [(c_{1},p_{1}) , \dots , (c_{i-1},p_{i-1}) , (c_{i+1},p_{i+1}) \dots , (c_{j},p_{j}) , (c_{\defsymbol},p)  ] \in \mG_{j-1}$ 
is equivalent to $R$, since $X(R,\dataset^{\prime}) = X(R^{\prime},\dataset^{\prime})$ for any $\dataset^{\prime}$. 
From these observations, the number of rules lists $R \in \mG_{j}$ with distinct projections on a dataset with $m$ elements cannot exceed $2^{j} \prod_{i=0}^{j-1}(d-i)$, as we consider all possible ordered choices of $j$ distinct elements from the set of $d$ features $\{ x_{1} , \dots , x_{d} \}$, and all $2^{j}$ possible assignments of the predicted values $\{ p_{1} , \dots , p_{j} \} \in \{0,1\}^{j}$, while $p$ is set to the unique value $\neq p_{j}$.
The sum of these bounds yields the statement. 
\end{proof}

A consequence of \Cref{thm:gfunct} is
the following upper bound to the VC-dimension $VC(\ruleset_{k}^{1})$ of $\ruleset_{k}^{1}$.
\begin{corollary}
\label{thm:vcupperbound}
The VC-dimension $VC(\ruleset_{k}^{1})$ of $\ruleset_{k}^{1}$ is 
\begin{align*}
VC(\ruleset_{k}^{1}) \leq \left\lfloor k \log_{2} \pars{2d } + 2 \right\rfloor.
\end{align*}
\end{corollary}
\begin{proof}
We first note that the VC-dimension of a range set $Q$ is at most $m$ if  it holds
$\log_{2}( \Lambda(Q , m) ) < m $. 
Using the upper bound to $\Lambda(\ruleset_{k}^{1} , m)$ from \Cref{thm:gfunct}, we have the following:
\begin{align*}
& \log_{2}  (\Lambda(\ruleset_{k}^{1} , m)) 
\leq  \log_{2} \pars{ 2 + \sum_{j=1}^{k} \sqpars{ 2^{j} \prod_{i=0}^{j-1}(d-i) } } \\
&\leq  \log_{2} \pars{ 2 + \sum_{j=1}^{k} (2d)^{j} } 
\leq  \log_{2} \pars{1 + \frac{(2d)^{k+1}  }{2d-1} } \\
& \leq  \log_{2} \pars{ \frac{(2d)^{k+1} }{2d-1} } + \log_{2}(e) \frac{2d-1}{(2d)^{k+1} } ,
\end{align*}
where in the last inequality we have used the fact that $\log_{2}(1+x)\leq \log_2(x) + \log_2(e)/x$, $\forall x \geq 0$.
Then, we have 
\begin{align*}
& \log_{2} \pars{ \frac{(2d)^{k+1} }{2d-1} } + \log_{2}(e) \frac{2d-1}{(2d)^{k+1} } \\
&\leq (k+1) \log_{2} \pars{ 2d } - \log_{2} \pars{ 2d-1 } + \log_{2}(e)( 2d)^{-k} \\
& = k \log_{2} \pars{ 2d } + \log_{2} \pars{ \frac{2d}{2d-1} } + \log_{2}(e)( 2d)^{-k}  \\
&\leq k \log_{2} \pars{2d} + 1 + \log_{2}(e)(2d)^{-k} 
< k \log_{2} \pars{2d} + 2 ,
\end{align*}
proving the statement. 
\end{proof}
Using the above upper bound to $VC(\ruleset_{k}^{1})$, we prove the following upper bound to the VC-dimension $VC(\ruleset_{k}^{z})$ of $\ruleset_{k}^{z}$.
\begin{corollary}
\label{thm:vcupperboundgen}
The VC-dimension $VC(\ruleset_{k}^{z})$ of $\ruleset_{k}^{z}$ is 
\begin{align*}
VC(\ruleset_{k}^{z}) \leq \left\lfloor k z \log_{2} \pars{ \frac{2 e d}{z} } + 2 \right\rfloor.
\end{align*}
\end{corollary}
The bounds we derived in \Cref{thm:vcupperbound} and \Cref{thm:vcupperboundgen} scale linearly with $k$ and $z$, and logarithmically with $d$. 
This implies that, since we are interested in \emph{sparse} models with small values of $k$ and $z$, the resulting complexity will not be large. 
This guarantees that accurate models will generalize well, and that the random sample $\sample$ should be representative of the dataset $\dataset$. 
Moreover, the weak dependence of the bounds on $d$ allows \algname\ to be extremely effective even on datasets with large sets of features. 

Rudin et al.~\cite{rudin2013learning} study generalization bounds for classes of binary classification models built from sequences of association rules; this general model includes rule lists. 
Moreover, they show that the VC-dimension of the set of rule lists created using pre-mined rules is exactly the size of the set of pre-mined rules, which is $\BOi{d^{z}}$ (see Theorem~3 in~\cite{rudin2013learning}). 
However, we remark that this result only holds for unconstrained rule lists (e.g., without any bound $k$ on their length),
and cannot be adapted to our setting. 
Our $\BOi{kz \log(d/z)}$ upper bound on the VC-dimension is more accurate, since it offers a more granular dependence on the parameters defining the rule list search space. 
More precisely, it explicitly depends on the maximum number $k$ of rules in the list, the number of features $d$, and the maximum number $z$ of conditions in each rule, while the bound $\BOi{d^{z}}$ from \cite{rudin2013learning} is not sensible to such constraints (as it assumes $k = d^{z}$). 
Since in most cases $k \ll d^{z}$, our upper bounds are orders of magnitude smaller. 

\subsubsection{Lower bounds to the VC-dimension}
In this section we prove almost matching \emph{lower bounds} to the VC-dimension of rule lists, 
confirming the tightness of our analysis. 
Our first result involves a lower bound to $VC(\ruleset_{k}^{1})$.
\begin{theorem}
\label{thm:vclowerbound}
The VC-dimension $VC(\ruleset_{k}^{1})$ of $\ruleset_{k}^{1}$ is 
\begin{align*}
VC(\ruleset_{k}^{1}) &\geq \left\lfloor k \log_{2} \pars{\frac{d+k}{k}} \right\rfloor.
\end{align*}
\end{theorem}
We may observe that the upper bound $\BOi{k \log (d)}$ given by \Cref{thm:vcupperbound} is almost tight, since \Cref{thm:vclowerbound} implies that $VC(\ruleset_{k}^{1}) \in \BLi{k \log (\frac{d+k}{k})}$. 
Reducing this gap is an interesting open problem. 
We remark that the lower bound from \Cref{thm:vclowerbound} improves the best known result, which provides a (weaker) lower bound of rate $\BLi{k + \log (d-k)}$ (see Theorem~3 of~\cite{yildiz2014vc}). 

We now prove a lower bound for the general case $z \ge 1$.
\begin{theorem}
\label{thm:vclowerboundgen}
The VC-dimension $VC(\ruleset_{k}^{z})$ of $\ruleset_{k}^{z}$ is 
\begin{align*}
VC(\ruleset_{k}^{z}) &\geq \left\lfloor k z \log_{2} \pars{\frac{d}{z \sqrt[z]{k} }} \right\rfloor.
\end{align*}
\end{theorem}
From \Cref{thm:vclowerboundgen} and \Cref{thm:vcupperboundgen} we conclude that 
$VC(\ruleset_{k}^{z}) \in \BOi{k z \log(d/z)}$ and 
$VC(\ruleset_{k}^{z}) \in \Omega( k z \log( d/z\sqrt[z]{k} ) )$, 
which are almost matching. 
A further refinement of either the upper or lower bound is an interesting direction for future research.

\subsection{Approximation guarantees}
\label{sec:guarantees}

In this section we prove sharp approximation bounds for the estimated losses $\ell(R,\sample)$ of rules $R$ on a random sample $\sample$ w.r.t. the losses $\ell(R,\dataset)$ on the dataset $\dataset$,
and derive sufficient sample sizes to compute $(\varepsilon, \theta)$-approximations for the optimal rule list $R^\star$. 
Due to space constraints, we provide most of the proofs in the Appendix. 
We define $\omega$ as 
$\omega = kz\ln(2ed/z)+2 $, a complexity parameter of our bounds derived from the results of \Cref{sec:VCrulelists} (\Cref{thm:vcupperboundgen}). 
\begin{theorem}
\label{thm:approxbounds}
Let $\sample$ be an i.i.d. random sample of size $m \geq 1$ of the dataset $\dataset$.
Then, for $\delta \in (0,1)$, the following inequalities hold simultaneously with probability $\geq 1-\delta$:
\begin{align*}
\ell(R,\dataset) &\leq \ell(R,\sample) \!+\! \sqrt{\frac{2 \ell(R,\sample) \! \pars{ \omega \! + \! \logtdelta } }{ m }} \!+\! \frac{2 \! \pars{ \omega \! + \! \logtdelta } }{ m } , \forall R \in \ruleset_{k}^{z}, \\
\ell(R^\star,\sample) &\leq \ell(R^\star,\dataset) + \sqrt{\frac{3 \ell(R^\star,\dataset) \logtdelta }{m}} .
\end{align*}
\end{theorem}
From the concentration bounds of \Cref{thm:approxbounds} we conclude that, with high probability, all estimates $\ell(R,\sample)$ from the sample $\sample$ provide accurate upper bounds to the losses $\ell(R,\dataset)$ in the dataset, while the loss $\ell(R^\star,\sample)$ of the optimal rule list strongly concentrates to its expectation $\ell(R^\star,\dataset)$. 
We now combine these inequalities to obtain an easy-to-check condition;
this condition yields a sufficient number of samples to obtain a high-quality approximation of the optimal rule list $R^\star$ from a random sample.
Define $\hat{m}(\ruleset_{k}^{z} , \dataset)$ as the minimum value of $m \geq 1$ such that
it holds 
\begin{align}
\sqrt{ \frac{ 3\theta\ln(\frac{2}{\delta})}{m} } 
+  \sqrt{ \frac{2 \Bigl( \theta \! + \! \sqrt{ \frac{ 3\theta\ln(\frac{2}{\delta})}{m} }  \Bigr) \bigl( \omega \! + \! \ln(\frac{2}{\delta}) \bigr) }{m} } 
+ \frac{2 \bigl( \omega \! + \! \ln(\frac{2}{\delta}) \bigr) }{m} 
\leq \varepsilon \theta . \label{eq:sampleboundexpl}
\end{align}
We prove that a sample $\sample$ of size at least $\hat{m}(\ruleset_{k}^{z} , \dataset)$ provides an $(\varepsilon, \theta)$-approximation of $R^\star$ with high probability.
\begin{theorem}
\label{thm:algoguarantees}
Let $\sample$ be an i.i.d. random sample of size $\geq \hat{m}(\ruleset_{k}^{z} , \dataset)$ of the dataset $\dataset$.
The output $\tilde{R} = \argmin_{R \in \ruleset_{k}^{z}} \ell(R,\sample) $ of \algname\ provides an $(\varepsilon, \theta)$-approximation of $R^\star$ with probability $\geq 1-\delta$.
\end{theorem}
\begin{proof}
Our proof is based on obtaining an upper bound to $\ell(\tilde{R},\dataset)$, where $\tilde{R}$ is the rule list reported by \algname, that is a function of the optimal loss $\ell(R^{\star},\dataset)$, and then showing that it is sufficiently small to guarantee an $(\varepsilon, \theta)$-approximation. 
First, we observe that $\ell(R^{\star} , \dataset) \leq \ell(\tilde{R} , \dataset)$, as $R^{\star}$ is one of the rule lists with minimum loss. 
Therefore, from the first set of inequalities of \Cref{thm:approxbounds} setting $R = \tilde{R}$, we have the upper bound to $\ell(R^{\star} , \dataset)$ 
\begin{align*}
\ell(R^{\star} , \dataset) 
\leq \ell(\tilde{R} , \sample) 
\!+\! \sqrt{\frac{2 \ell(\tilde{R},\sample) \! \pars{ \omega \! + \! \logtdelta } }{ m }} \!+\! \frac{2 \! \pars{ \omega \! + \! \logtdelta } }{ m } . \numberthis \label{eq:proofone}
\end{align*}
We now prove that, since $\tilde{R}$ is chosen as the rule list with minimum loss on the sample, its estimated loss $\ell(\tilde{R} , \sample)$ should be concentrated towards $\ell(R^{\star} , \dataset)$. 
We have $\ell(\tilde{R} , \sample) \leq \ell(R^{\star} , \sample)$, as $\tilde{R}$ is one of the rule list with minimum loss in $\sample$.
Using the last inequality of \Cref{thm:approxbounds} it holds
\begin{align}
\ell(\tilde{R} , \sample) \leq \ell(R^{\star} , \sample) 
\leq \ell(R^{\star} , \dataset) + \sqrt{\frac{3 \ell(R^\star,\dataset) \logtdelta }{m}} . \label{eq:prooftwo}
\end{align}
Replacing all occurences of $\ell(\tilde{R} , \sample)$ in \eqref{eq:proofone} with its upper bound given by the r.h.s. of \eqref{eq:prooftwo}, we obtain 
\begin{align}
\ell(\tilde{R} , \dataset) \leq \ell(R^{\star} , \dataset) + u( \ell(R^{\star} , \dataset) , m ) ,
\end{align}
where the function $u(v , m)$ is 
\begin{align*}
u(v , m) \! = \! \sqrt{\frac{3 v \logtdelta }{m}} \! + \! \sqrt{\frac{2 \Bigl( v \! + \! \sqrt{\frac{3 v \logtdelta }{m} \Bigr) } \! \pars{ \omega \! + \! \logtdelta } }{ m }} \!+\! \frac{2 \pars{ \omega \! + \! \logtdelta } }{ m } .
\end{align*}
We now seek to identify the minimum $m$ such that $\ell(\tilde{R} , \dataset)$ is guaranteed to be $\leq \ell(R^{\star} , \dataset) + \varepsilon \max\{ \theta , \ell(R^{\star} , \dataset) \}$. 
Note that, from the derivations above, it is sufficient to verify that $u(\ell(R^{\star} , \dataset)  , m) \leq \varepsilon \max\{ \theta , \ell(R^{\star} , \dataset) \}$, for any $\ell(R^{\star} , \dataset) \in [0,1]$.
First, consider the case $\ell(R^{\star} , \dataset) \leq \theta$.
We have that
\begin{align*}
u(\ell(R^{\star} , \dataset)  , m) \leq u(\theta  , m) \leq \varepsilon \theta
\end{align*}
is true when $m \geq \hat{m}(\ruleset_{k}^{z} , \dataset)$ by definition of $\hat{m}(\ruleset_{k}^{z} , \dataset)$ given in the statement, since $u(\theta  , m) \leq \varepsilon\theta$ is guaranteed by \eqref{eq:sampleboundexpl}. 
We now consider the case $\ell(R^{\star} , \dataset) \geq \theta$.
We need to verify that 
\begin{align}
u(v  , m) \leq \varepsilon v , \forall  \theta \leq v \leq 1 . \label{eq:condtocheck}
\end{align}
From the definition of $\hat{m}(\ruleset_{k}^{z} , \dataset)$, we know that $u(\theta  , m) \leq \varepsilon \theta$ holds for all $m \geq \hat{m}(\ruleset_{k}^{z} , \dataset)$, therefore \eqref{eq:condtocheck} is verified for $v=\theta$.
By taking the derivative w.r.t. $v$ of both sides of \eqref{eq:condtocheck}, it is simple to check that the derivative of the r.h.s. is constant, while the derivative of the l.h.s. is monotonically decreasing with $v$. 
Therefore, \eqref{eq:condtocheck} also holds for all $\theta < v \leq 1$,
obtaining the statement. 
\end{proof}
We note that, while $\hat{m}(\ruleset_{k}^{z} , \dataset)$ can be easily computed (e.g., with a binary search over $m$), from \eqref{eq:sampleboundexpl} it may not be simple to interpret its dependence on the several parameters. To do so, we prove the following upper bound to $\hat{m}(\ruleset_{k}^{z} , \dataset)$.
\begin{theorem}
\label{thm:vcupperboundbigo}
It holds 
$\hat{m}(\ruleset_{k}^{z} , \dataset) \in \mathcal{O} \bigl( \frac{\omega + \ln(\frac{1}{\delta})}{\varepsilon^2 \theta} \bigr) $.
\end{theorem}
\begin{proof}
First, we prove that $\hat{m}(\ruleset_{k}^{z} , \dataset) \geq 3\ln(2/\delta) / \theta$. 
This follows easily from $\sqrt{\frac{3\theta\ln(2/\delta)}{m}} \leq \varepsilon\theta$, i.e., only considering the leftmost term of \eqref{eq:sampleboundexpl}. 
Then, it holds
\begin{align*}
& \sqrt{ \frac{ 3\theta\ln(\frac{2}{\delta})}{m} } + \sqrt{ \frac{2 \pars{ \theta \! + \! \sqrt{ \frac{ 3\theta\ln(\frac{2}{\delta})}{m} } } \bigl( \omega \! + \! \ln(\frac{2}{\delta}) \bigr) }{m} } + \frac{2 \bigl( \omega \! + \! \ln(\frac{2}{\delta}) \bigr) }{m} \\
&\leq \sqrt{ \frac{ 3\theta\ln(\frac{2}{\delta})}{m} } + \sqrt{ \frac{ 4\theta \bigl( \omega \! + \! \ln(\frac{2}{\delta}) \bigr)}{m} }  + \frac{2 \bigl( \omega \! + \! \ln(\frac{2}{\delta}) \bigr) }{m}  \\
&\leq \sqrt{ \frac{ 14 \theta \bigl( \omega \! + \! \ln(\frac{2}{\delta}) \bigr)}{m} }  + \frac{2 \bigl( \omega \! + \! \ln(\frac{2}{\delta}) \bigr)}{m} ,
\end{align*}
where the second-last inequality holds for all $m \geq 3\ln(2/\delta) / \theta$. 
By solving the quadratic inequality
\begin{align*}
\sqrt{ 14 m \theta (\omega + \ln(2/\delta)) }  + 2 (\omega + \ln(2/\delta)) \leq m \varepsilon \theta ,
\end{align*}
we observe that \eqref{eq:sampleboundexpl} holds for all $m$ with
\begin{align}
m \geq \frac{  \bigl( \omega \! + \! \ln(\frac{2}{\delta}) \bigr) \sqrt{14} \sqrt{8 \varepsilon \! + \! 14}   }{ 2 \varepsilon^2 \theta }
+ \frac{2 \bigl( \omega \! + \! \ln(\frac{2}{\delta}) \bigr) }{ \varepsilon \theta }
+ \frac{7 \bigl( \omega \! + \! \ln(\frac{2}{\delta}) \bigr) }{ \varepsilon^2 \theta } . \label{eq:mhatupperbound}
\end{align}
The r.h.s. of \eqref{eq:mhatupperbound} is $\in \BOi{ \frac{\omega + \ln(\frac{1}{\delta})}{\varepsilon^2 \theta} }$; 
moreover, it upper bounds $\hat{m}(\ruleset_{k}^{z} , \dataset)$, as \eqref{eq:mhatupperbound} provides a value of $m$ that make \eqref{eq:sampleboundexpl} true (not necessarily the minimum), obtaining the statement. 
\end{proof}
Interestingly, the sample size $\hat{m}(\ruleset_{k}^{z} , \dataset)$ is completely independent of the size $n$ of the dataset $\dataset$: 
it only depends on the parameters $\omega = kz\ln(2ed/z)+2$ of the rule lists search space, and the desired approximation accuracy $\varepsilon, \theta$ and confidence $\delta$. 
This characteristic allows \algname\ to be applied to massive datasets, of arbitrarily large size $n$. 

\section{Experiments}
\label{sec:experiments}

This section presents the results of our experiments.
The main goal of our experimental evaluation is to test the scalability of \algname\ in analyzing large datasets compared to exact approaches (\Cref{sec:expexact}). 
To do so, we measure the number of samples used by \algname\ to obtain an accurate approximation of the best rule, its running time,
and the accuracy of the reported rule list compared to the optimal solution. Then, we also compare \algname\ with state-of-the-art heuristics for rule list training (\Cref{sec:expheuristics}),
as such methods, while not offering theoretical guarantees, may still provide accurate rule lists in practice.
Finally, we quantify the performance of \algname\ under several settings of the rule list parameters $z$ and $k$ (\Cref{sec:expparams}).

\textit{Datasets.}
We tested \algname\ on $8$ benchmark datasets from UCI\footnote{\url{https://archive.ics.uci.edu/}}.
We binarized the datasets containing countinous features by considering $4$ thresholds at equally spaced quantiles: 
for each countinous feature $f$ and each threshold $t$, we created two binary features $\qtm{f \geq t}$ and $\qtm{f < t}$. 
The statistics of the resulting binary datasets are described in \Cref{tab:datasets}. 
Since some of these datasets are quite small, we replicate them $r$ times, i.e., each training sample of $\dataset$ is copied $r$ times (see the column $r$ of \Cref{tab:datasets}). 
Note that this preprocessing allows to obtain larger datasets, all with approximately $10^6$ training instances, while preserving the rule list distribution and search space structure, as the losses and covered instances of all rule lists are the same on the original and replicated datasets. 

\textit{Compared methods.} 
Our main goal is to compare \algname\ with CORELS~\cite{angelino2018learning}, the state-of-the-art exact method to identify the rule list with minimum loss from a dataset $\dataset$. 
\algname, as described in \Cref{sec:algo}, generates the random sample $\sample$ and then runs an exact algorithm to search for the rule list with minimum loss on $\sample$ (using the same settings and parameters). Note that by using CORELS as exact algorithm within \algname, all differences between \algname\ and CORELS are due to the use of sampling, i.e., there are no other confounding factors in the comparison between the two methods. 
We also compare \algname\ with SBRL~\cite{yang2017scalable} and RIPPER~\cite{cohen1995fast}, two state-of-the-art heuristic approaches for rule list learning. 
SBRL uses a scalable Monte-Carlo approach to approximately search for accurate rule lists, while RIPPER leverages a greedy selection of conditions. 
Our comparison is motivated by the fact that, even if these methods do not provide guarantees in terms of solution quality, they are often the methods of choice for practitioners since they obtain good solutions and are designed to scale to large datasets.

\textit{Rule Lists Parameters.}
To compare \algname\ with exact and heuristic methods (Sections~\ref{sec:expexact} and~\ref{sec:expheuristics}), for each dataset we set the parameter $k$ as shown in \Cref{tab:datasets} and $z=1$. 
In Section~\ref{sec:expparams} we evaluate the impact of different choices of $k$ and $z$. 
Regarding \algname, we compute $(\varepsilon, \theta)$-approximations for several combinations of the parameters $\varepsilon$ and $\theta$.
We consider $\varepsilon \in \{ 1 , 0.5 , 0.25 \}$, and vary $\theta$ in the interval $[0.005, 0.05]$. 
For all experiments we fix $\delta = 0.05$, as we did not observe significant differences for other values (given the exponential dependence of the bounds w.r.t. $\delta$).

\textit{Experimental setup.}
We implemented \algname\ in Python.
The code and the scripts to reproduce all experiments are available online\footnote{\url{https://github.com/VandinLab/SamRuLe}}.
For CORELS, we have used the implementation available online\footnote{\url{https://github.com/corels/corels}}.
We made minor modifications to the original implementation to limit its exploration to rule lists of length at most $k$ 
(i.e., pruning all rules of length $>k$ instead of performing an unbounded search). 
We made similar minor changes to the implementation of SBLR\footnote{\url{https://github.com/Hongyuy/sbrlmod/}}. 
We evaluated RIPPER with a recent efficient implementation\footnote{\url{https://github.com/imoscovitz/wittgenstein}}.
All the code was compiled and executed on a machine 
equipped with 2.30 GHz Intel Xeon CPU, 
 $1$ TB of RAM, on Ubuntu 20.04. 
We repeated all experiments $10$ times, and report averages $\pm $ stds 
over the $10$ repetitions.

\begin{table}
  \caption{Statistics of the datasets considered in our experiments.  
  $n$ is the number of transactions, $d$ is the number of binary features, $r$ is the replication factor, $k$ is the maximum rule list length.  }
\label{tab:datasets}
\center
  \begin{tabular}{lrrrrrr}
    \toprule
    $\dataset$              & $n$      & $d$ & $r$  & $k$     \\
    \midrule
	a9a     & 32561 & 124  & 100 & 4 \\
	adult     & 32561 & 175  & 100 & 4 \\
	bank     & 41188 &  152 & 100 & 4 \\
	higgs     & $11000000$ &  531 & 1 & 3 \\
	ijcnn1     & 91701 &  35 & 100 & 8 \\
	mushroom     & 8124 & 118 & 200 & 5  \\
	phishing     & 11050 &  69 & 100 & 8 \\
	susy     & $5000000$ &  179 & 1 & 5 \\
  \bottomrule
\end{tabular}
\end{table}

\begin{figure*}[ht]
\begin{subfigure}{.525\textwidth}
  \centering
  \includegraphics[width=\textwidth]{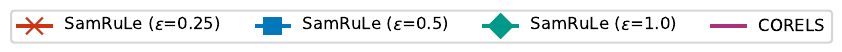}
\end{subfigure} \\
\begin{subfigure}{.247\textwidth}
  \centering
  \includegraphics[width=\textwidth]{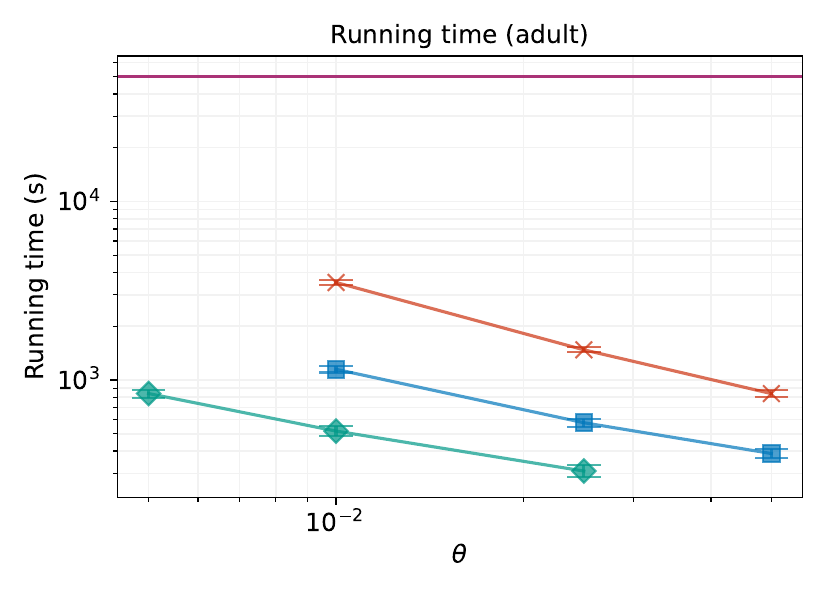}
  \caption{}
\end{subfigure}
\begin{subfigure}{.247\textwidth}
  \centering
  \includegraphics[width=\textwidth]{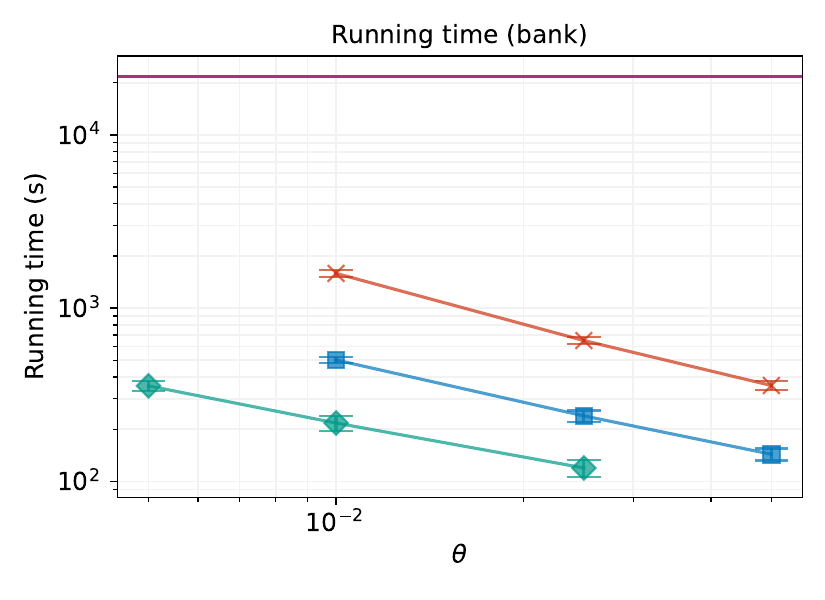}
  \caption{}
\end{subfigure}
\begin{subfigure}{.247\textwidth}
  \centering
  \includegraphics[width=\textwidth]{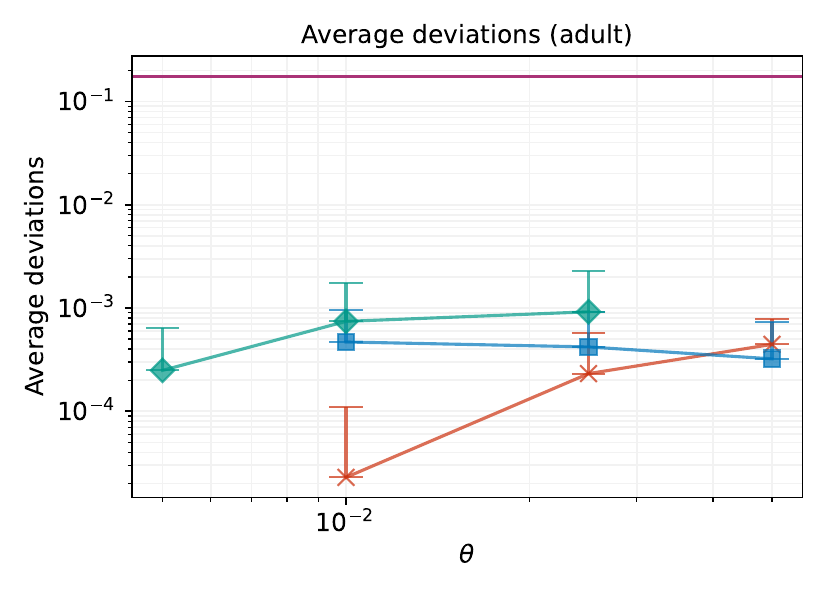}
  \caption{}
\end{subfigure}
\begin{subfigure}{.247\textwidth}
  \centering
  \includegraphics[width=\textwidth]{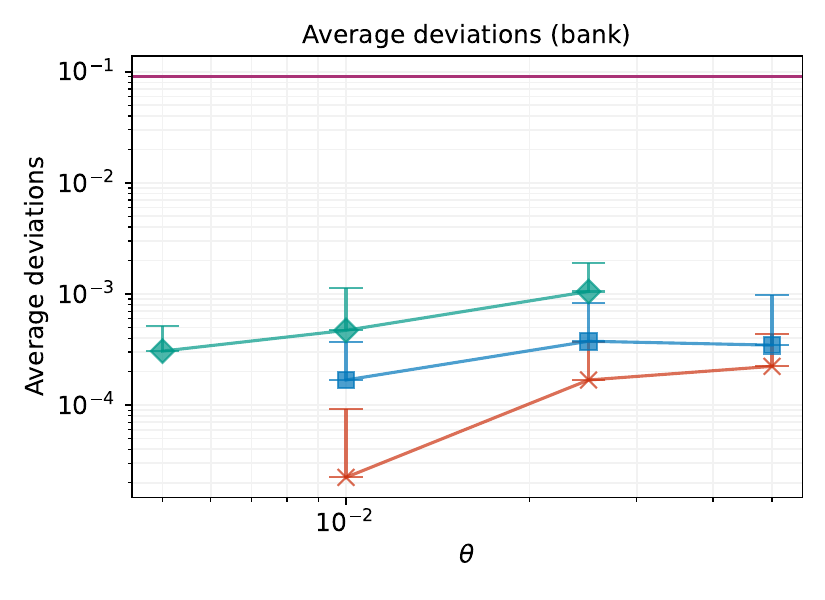}
  \caption{}
\end{subfigure}
\caption{ 
Performance and accuracy comparison between \algname\ and CORELS on adult and bank datasets, for different values of $\varepsilon$ and $\theta$. 
(a)-(b): running times of CORELS and \algname. 
(c)-(d): average deviations $| \ell(\tilde{R} , \dataset) - \ell(R^\star , \dataset) |$ of the loss of the rule list $\tilde{R}$ found by \algname\ with the optimal rule list $R^\star$ found by CORELS (purple horizontal line drawn at $y = \ell(R^\star , \dataset)$). The deviation plots only show upper errors bars at $+$std to improve readability. See Figures \ref{fig:runningtimes} and \ref{fig:avgdevs} in the Appendix for the plots for all datasets and with $\pm$std bars.
}
\label{fig:runningtimeavgdevsmain}
\Description{The figures show a performance and accuracy comparison between SamRuLe and CORELS on the adult and bank datasets, for different values of the parameters. SamRuLe needs a fraction of the time of CORELS, since it considers a small random sample of the large dataset.}
\end{figure*}

\begin{figure*}[ht]
\begin{subfigure}{.22\textwidth}
\begin{adjustbox}{varwidth=\textwidth,fbox,center}
\begin{subfigure}{\textwidth}
\footnotesize
\tikzset{every picture/.style={line width=0.75pt}}
\begin{tikzpicture}[x=0.75pt,y=0.75pt,yscale=-1,xscale=1]
\draw (150,125) node [anchor=north west][inner sep=0.75pt]   [align=left] {{ \textbf{if} capital-gain $\geq 2 \cdot 10^4 \rightarrow 1$}};
\draw (150,140.5) node [anchor=north west][inner sep=0.75pt]   [align=left] {{ \textbf{else if} capital-loss $\geq 1742 \rightarrow 1$}};
\draw (150,155) node [anchor=north west][inner sep=0.75pt]   [align=left] {{ \textbf{else if} education-num $< 13 \rightarrow 0$}};
\draw (150,170) node [anchor=north west][inner sep=0.75pt]   [align=left] {{ \textbf{else if} marital-status=married $\rightarrow 1$}};
\draw (150,185) node [anchor=north west][inner sep=0.75pt]   [align=left] {{ \textbf{else} $\rightarrow 0$}};
\end{tikzpicture}
\end{subfigure}
\end{adjustbox}
\caption{}
\end{subfigure}
\begin{subfigure}{.72\textwidth}
\begin{adjustbox}{varwidth=\textwidth,fbox,center}
\begin{subfigure}{.31\textwidth}
\footnotesize
\tikzset{every picture/.style={line width=0.75pt}}
\begin{tikzpicture}[x=0.75pt,y=0.75pt,yscale=-1,xscale=1]
\draw (150,125) node [anchor=north west][inner sep=0.75pt]   [align=left] {{ \textbf{if} capital-gain $\geq 2 \cdot 10^4 \rightarrow 1$}};
\draw (150,140.5) node [anchor=north west][inner sep=0.75pt]   [align=left] {{ \textbf{else if} capital-loss $\geq 1742 \rightarrow 1$}};
\draw (150,155) node [anchor=north west][inner sep=0.75pt]   [align=left] {{ \textbf{else if} education-num $< 13 \rightarrow 0$}};
\draw (150,170) node [anchor=north west][inner sep=0.75pt]   [align=left] {{ \textbf{else if} marital-status=married $\rightarrow 1$}};
\draw (150,185) node [anchor=north west][inner sep=0.75pt]   [align=left] {{ \textbf{else} $\rightarrow 0$}};
\end{tikzpicture}
\end{subfigure}
\begin{subfigure}{.31\textwidth}
\footnotesize
\tikzset{every picture/.style={line width=0.75pt}}
\begin{tikzpicture}[x=0.75pt,y=0.75pt,yscale=-1,xscale=1]
\draw (150,125) node [anchor=north west][inner sep=0.75pt]   [align=left] {{ \textbf{if} capital-gain $\geq 2 \cdot 10^4 \rightarrow 1$}};
\draw (150,140.5) node [anchor=north west][inner sep=0.75pt]   [align=left] {{ \textbf{else if} age $< 26 \rightarrow 0$}};
\draw (150,155) node [anchor=north west][inner sep=0.75pt]   [align=left] {{ \textbf{else if} education-num $< 13 \rightarrow 0$}};
\draw (150,170) node [anchor=north west][inner sep=0.75pt]   [align=left] {{ \textbf{else if} marital-status=married $\rightarrow 1$}};
\draw (150,185) node [anchor=north west][inner sep=0.75pt]   [align=left] {{ \textbf{else} $\rightarrow 0$}};
\end{tikzpicture}
\end{subfigure}
\begin{subfigure}{.31\textwidth}
\footnotesize
\tikzset{every picture/.style={line width=0.75pt}}
\begin{tikzpicture}[x=0.75pt,y=0.75pt,yscale=-1,xscale=1]
\draw (150,125) node [anchor=north west][inner sep=0.75pt]   [align=left] {{ \textbf{if} capital-gain $\geq 2 \cdot 10^4 \rightarrow 1$}};
\draw (150,140.5) node [anchor=north west][inner sep=0.75pt]   [align=left] {{ \textbf{else if} hours-per-week $< 35 \rightarrow 0$}};
\draw (150,155) node [anchor=north west][inner sep=0.75pt]   [align=left] {{ \textbf{else if} education-num $< 11 \rightarrow 0$}};
\draw (150,170) node [anchor=north west][inner sep=0.75pt]   [align=left] {{ \textbf{else if} marital-status=married $\rightarrow 1$}};
\draw (150,185) node [anchor=north west][inner sep=0.75pt]   [align=left] {{ \textbf{else} $\rightarrow 0$}};
\end{tikzpicture}
\end{subfigure}
\end{adjustbox}
\caption{}
\end{subfigure}
\caption{(a): optimal rule $R^\star$ computed by CORELS on the adult dataset ($\ell(R^\star,\dataset)=0.176$) to predict high income (the label $1$ denotes $\qtm{\geq 50K}$).
(b): set of rule lists computed by \algname\ over $10$ runs. 
\algname\ identified the optimal rule  
and slight variations $\tilde{R}_1$ and $\tilde{R}_2$
that differ in the second rule of the list: 
they predict a lower outcome using the age ($\tilde{R}_1$) and the per-week work hours features ($\tilde{R}_2$) with respective loss $\ell(\tilde{R}_1,\dataset)=0.1763$ and $\ell(\tilde{R}_2,\dataset)=0.1775$. }
\label{fig:adultrules}
\Description{The figures show the optimal rule list found by CORELS on the adult dataset, and the rule lists reported by SamRuLe from the random samples. The reported models by SamRuLe have an accuracy very close to the optimal, and are composed by very similar logical conditions.}
\end{figure*}

\subsection{Comparison to exact method}
\label{sec:expexact}
In this section we describe the experimental comparison between \algname\ and CORELS, the state-of-the-art method to identify the optimal rule list $R^{\star}$ with minimum loss. 
Our main goal is to evaluate the scalability of \algname\ in terms of number of samples and running time required to obtain an accurate approximation of the best rule. 
Furthermore, we evaluate, both quantitatively and qualitatively, the accuracy of the rule list found by \algname\ w.r.t. the optimal solution returned by CORELS. 
We ran both methods on all datasets, setting $z=1$ and $k$ as in \Cref{tab:datasets}.
For \algname, we vary the parameters $\theta$ and $\varepsilon$ as described at the beginning of \Cref{sec:experiments}.

\Cref{fig:runningtimeavgdevsmain} shows the results for these experiments. 
In \Cref{fig:runningtimeavgdevsmain}.(a) and (b) we compare the running time of \algname\ with the time needed by CORELS on the adult and bank datasets.
The results for other datasets are very similar, and shown in \Cref{fig:runningtimes} (in the Appendix).
From these results, we can immediately conclude that \algname\ requires a \emph{small fraction} of the time needed by CORELS, with an improvement of up to $2$ orders of magnitude. 
The reason for this significant speedup is that \algname\ searches for the rule with minimum loss on a small sample $\sample$, which is all cases orders of magnitude smaller than the size of original dataset $\dataset$. 
We show the number of samples $\hat{m}(\ruleset_{k}^{z} , \dataset)$ used by \algname\ for all datasets and for all parameters in \Cref{fig:samplesizes} (in the Appendix).

Then, we evaluated the quality of the solutions returned by \algname\ with the optimal rule list computed by CORELS. 
Figures~\ref{fig:runningtimeavgdevsmain}.(c) and (d) show the average deviations $| \ell(\tilde{R} , \dataset) - \ell(R^\star , \dataset) |$ of the loss of the rule list $\tilde{R}$ found by \algname\ w.r.t. the optimal rule list $R^\star$ found by CORELS on the dataset $\dataset$. 
To verify the validity of \algname's theoretical guarantees, the plots also show the loss of the optimal solution $\ell(R^\star , \dataset)$ found by CORELS (purple horizontal line), and upper error bars ($+$std) for the average deviations (see Figures \ref{fig:runningtimes} and \ref{fig:avgdevs} in the Appendix for the plots for all datasets with both bars). 
From these results, we observe that the rule lists reported by \algname\ are extremely accurate in terms of prediction accuracy, since the deviations $| \ell(\tilde{R} , \dataset) - \ell(R^\star , \dataset) |$ are \emph{orders of magnitude} smaller than $\ell(R^\star , \dataset)$, and smaller than guaranteed by our theoretical analysis. 
This confirms that \algname\ outputs extremely accurate rule lists, even when trained on random samples that are orders of magnitude smaller than the entire dataset. 
Furthermore, it is likely that the guarantees of the $(\varepsilon, \theta)$-approximations hold for samples smaller than what guaranteed by our analysis; 
this leaves significant opportunities for further improvements of our algorithm. 
Then, we quantify the deviations between the loss $\ell(\tilde{R} , \sample)$ of $\tilde{R}$ estimated on the sample w.r.t. to the loss $\ell(\tilde{R} , \dataset)$ on the dataset, i.e., the approximation error incurred by \algname\ due to analyzing $\sample$ instead of the entire dataset $\dataset$.
\ifextversion
In \Cref{fig:avgdevssampleloss} 
\fi
we show the average loss approximation error $| \ell(\tilde{R} , \dataset) - \ell(\tilde{R} , \sample) |$, which is also extremely small (i.e., $1$ to $2$ orders of magnitude smaller than $\ell(\tilde{R} , \dataset)$), confirming that the conclusions that can be drawn from the sample using \algname, e.g., from the estimated loss $\ell(\tilde{R} , \sample)$, are very close to the corresponding exact ones. 
Finally, we compared the logical conditions in the rule lists returned by \algname\ with the ones in the best rule list computed by CORELS. 
Our goal is to verify that the insights gained from the approximated prediction models from \algname\ were similar to the optimal ones, i.e., that \algname\ allows a qualitative interpretation of the reported rule list that was stable over the different experimental runs and similar to what obtainable from the exact analysis. 
\Cref{fig:adultrules} reports the optimal rule list computed by CORELS on the dataset adult (a), and the set of rule lists computed by \algname\ (b) over all runs.
Interestingly, we observe that the approximations from \algname\ either match the optimal solution, or are very similar to it. 
In fact, all rules reported by \algname\ either share the same conditions found in the optimal rule list, or replace one of the feature with alternative reasonable insights (e.g., predicting a lower income for young individuals and limited weekly working hours). 
In general, we found the solutions reported by \algname\ to be extremely stable and similar to the respective optimal solutions also for all other dataset. 

From these observations we conclude that \algname\ computes extremely accurate rule lists using a fraction of the resources needed by exact approaches, therefore scaling effectively to large datasets.

\begin{figure*}[ht]
\begin{subfigure}{.76\textwidth}
  \centering
  \includegraphics[width=.35\textwidth]{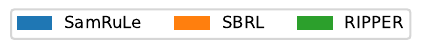}
\end{subfigure}
\begin{subfigure}{.23\textwidth}
\textcolor{white}{\rule{2cm}{0.01cm}}
\end{subfigure} \\
\begin{subfigure}{.365\textwidth}
  \centering
  \includegraphics[width=\textwidth]{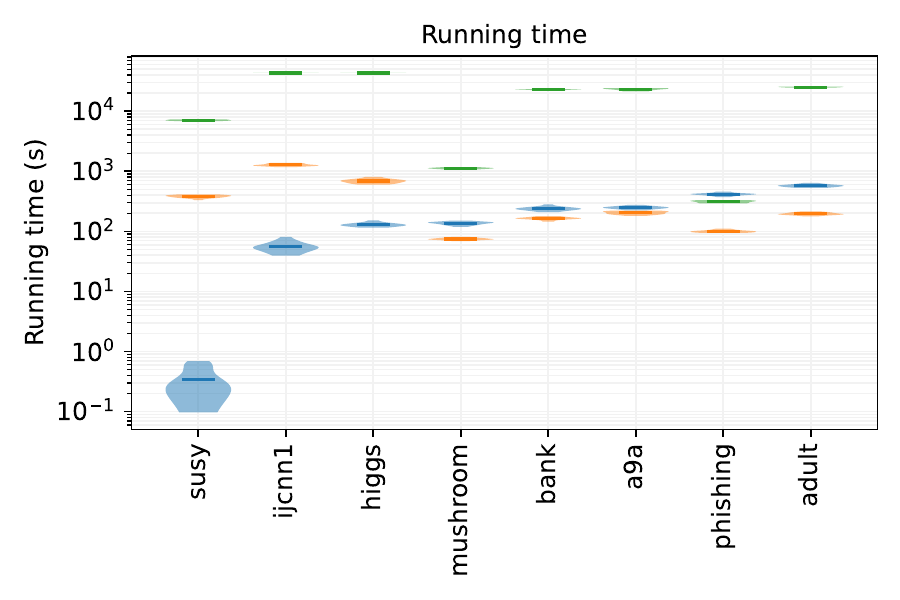}
  \caption{}
\end{subfigure}
\begin{subfigure}{.365\textwidth}
  \centering
  \includegraphics[width=\textwidth]{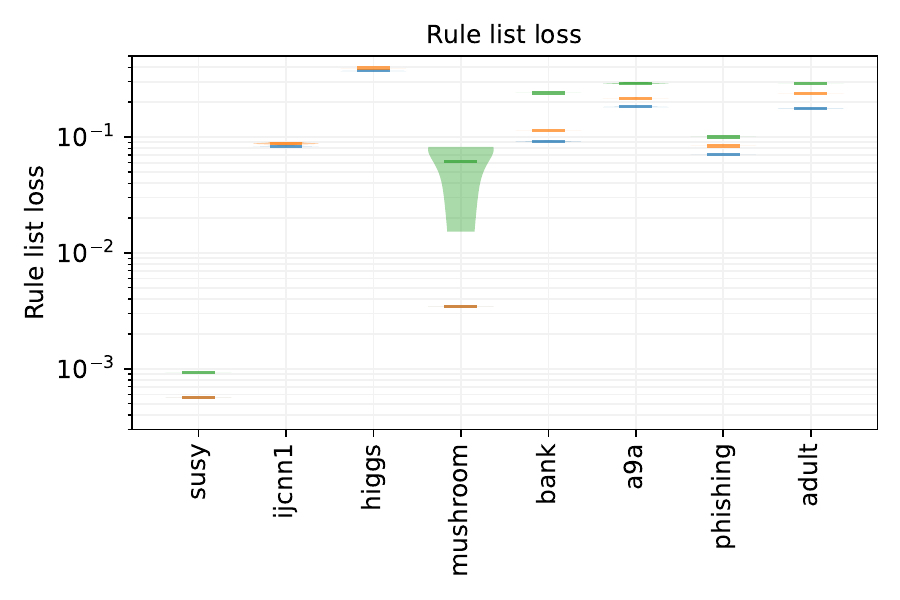}
  \caption{}
\end{subfigure}
\begin{subfigure}{.23\textwidth}
\begin{adjustbox}{varwidth=.92\textwidth,fbox,center}
\begin{subfigure}{\textwidth}
\footnotesize
\tikzset{every picture/.style={line width=0.75pt}}
\begin{tikzpicture}[x=0.75pt,y=0.75pt,yscale=-1,xscale=1]
\draw (150,125) node [anchor=north west][inner sep=0.75pt]   [align=left] {{ \textbf{if} age $\geq 26 \rightarrow 0$}};
\draw (150,140.5) node [anchor=north west][inner sep=0.75pt]   [align=left] {{ \textbf{else if} education-num $< 9 \rightarrow 0$}};
\draw (150,155) node [anchor=north west][inner sep=0.75pt]   [align=left] {{ \textbf{else if} marital-status=married $\rightarrow 1$}};
\draw (150,170) node [anchor=north west][inner sep=0.75pt]   [align=left] {{ \textbf{else if} education-num $< 13 \rightarrow 0$}};
\draw (150,185) node [anchor=north west][inner sep=0.75pt]   [align=left] {{ \textbf{else} $\rightarrow 0$}};
\end{tikzpicture}
\end{subfigure}
\end{adjustbox}
\caption{}
\end{subfigure}
\caption{ 
Comparison in terms of running time (a) and accuracy (b) between \algname, SBRL, and RIPPER. 
(c): rule $R$ computed by SBRL on adult with loss $\ell(R,\dataset)=0.238$ over all $10$ runs. 
}
\label{fig:compheuristics}
\Description{Comparison between the running time of SamRuLe and the heuristic methods SBRL and RIPPER. The time needed by SamRuLe is comparable to the heuristics, but the rule list loss it reports is always better. Furthermore, the heuristic may miss some key logical conditions that are, instead, always reported by SamRuLe.}
\end{figure*}

\subsection{Comparison to heuristic methods}
\label{sec:expheuristics}
In this set of experiments, we compare \algname\ with two state-of-the-art heuristic methods SBRL and RIPPER. 
For \algname\ we fix $\theta = 0.025$ and $\varepsilon=0.5$.
We ran SBRL for $10^4$ iterations using default parameters (as suggested by \cite{yang2017scalable,angelino2018learning}). 
Also for these experiments, for each dataset we fix $z=1$ and $k$ to the values in \Cref{tab:datasets}.

We show the results of these experiments in \Cref{fig:compheuristics}.

\Cref{fig:compheuristics}.(a) shows the running times for the three methods.
We observe that in all but one case, RIPPER is the slowest method. 
On two datasets (ijcnn1 and higgs) we stopped it as it could not complete after more than $12$ hours, while requiring a very large memory footprint (e.g., more than $400$ GB of memory for higgs). 
Regarding SBRL and \algname, both methods are fast, e.g., requiring always less than $18$ minutes.
More precisely, \algname\ is faster on $3$ datasets, up to $3$ orders of magnitude for the susy dataset. 
In other $2$ datasets, the running times of the two methods are comparable, while for phishing and adult SBRL is faster by a factor at most $4$.
This experiment confirms that \algname\ is very practical, requiring a lower or comparable amount of resources of state-of-the-art scalable heuristics.

\Cref{fig:compheuristics}.(b) compares all methods in terms of accuracy.
The plots show the values of the loss $\ell(\tilde{R},\dataset)$ for the rule list $\tilde{R}$ reported by all runs of the methods.
We may observe that RIPPER is the worst approach, as it always reports rule lists with the highest loss (results for ijcnn1 and higgs are not shown as RIPPER could not complete in reasonable time, as discussed before),
while \algname\ always provides a rule list with the smallest loss.
Regarding SBRL, we observe that it reports a rule list with the same, or almost the same, loss of \algname\ on $4$ datasets, while it provides suboptimal solutions for other cases, with losses up to $30\%$ higher than \algname. 
This suggests that, while SBRL scales to large datasets, it often provides solutions that are sensibly less accurate than the optimal one. 
Instead, as discussed previously, \algname\ outputs a rule list with guaranteed gap with the optimal solution, and always very close to it in practice. 
We remark that providing a suboptimal solution may also impact the interpretation for the predictions, e.g., missing relevant factors for the model. 
In fact, we report the rule list computed by SBRL on adult (see \Cref{fig:compheuristics}.(c)) for all the $10$ runs: 
such rule list does not involve the capital-gain feature, a key condition for the optimal solution  to predict high income (\Cref{fig:adultrules}.(a)), thus obtaining a higher loss ($0.238$).
In contrast, this feature was always found in the rule lists reported by \algname\ (\Cref{fig:adultrules}.(b)), which have loss always very close to the optimal ($0.176$).

Overall, compared to RIPPER and SBRL, \algname\ outputs an equally accurate solution in less time, or a sensibly better rule list using comparable resources. 
We conclude that \algname\ provides an excelled combination of scalability and high accuracy for rule list learning from large datasets, achieving a better trade-off than state-of-the-art heuristic methods with no theoretical guarantees. 

\subsection{Impact of rule list parameters}
\label{sec:expparams}
In this final set of experiments, we evaluate the impact to the performance of \algname\ of the rule list search space parameters $z$ and $k$. 
We focus on the datasets mushroom and phishing, as the results for other datasets were similar. 
We test all values of $1 \leq z \leq 3$ and $1 \leq k \leq 5$, fixing $\theta = 0.025$ and $\varepsilon = 0.5$, measuring the number of samples required by \algname, its running time, and the accuracy of the rule lists provided in output.

We show these results in 
\ifextversion
\Cref{fig:resultsparameters}. 
\fi
When increasing $k$ and $z$, the number of samples $\hat{m}(\ruleset_{k}^{z} , \dataset)$ considered by \algname\ grows 
following the expect trend $\hat{m}(\ruleset_{k}^{z} , \dataset) \in \BOi{(\omega + \log(1/\delta)) /(\varepsilon^{2} \theta)}$ proved in our analysis (\Cref{thm:vcupperboundbigo}),  
resulting in sample sizes that are always a small fraction of the size of the dataset.  
Consequently, the running time of \algname\ increases roughly linearly with the sample size, remaining practical for all settings. 
Regarding the accuracy of the rule lists, we observed the parameter $k$ to have the largest impact on the loss, that remains fairly stable w.r.t. $z$. 

These results demonstrate that \algname\ is applicable to complex analysis involving larger values of $k$ and $z$ while scaling to large datasets, that are in most cases out of reach of exact approaches. 

\section{Conclusions}
We introduced \algname, a novel and scalable algorithm to find nearly optimal rule lists. \algname\ uses sampling to scale to large datasets, and provides rigorous guarantees on the quality of the rule lists it reports. Our approach builds on the VC-dimension of rule lists, for which we proved novel upper and lower bounds. Our experimental evaluation shows that 
\algname\ enables learning highly accurate rule lists on large datasets, is up to two orders of magnitude faster than state-of-the-art exact approaches, and is as fast as, and sometimes faster than, recent heuristic approaches, while reporting higher quality rule lists.

Our work opens several interesting directions for future research, including the use of sampling to scale approaches for learning other rule-based models while providing rigorous guarantees on the quality of the learned model.
Moreover, the efficient computation of advanced \emph{data-dependent} complexity measures, such as Rademacher averages \cite{BartlettM02,pellegrina2022mcrapper,centrakdd}, may be useful to obtain even sharper approximation guarantees for our problem.

\begin{acks}
This work was supported by the \qt{National Center for HPC, Big Data, and Quantum Computing}, project CN00000013, and by the PRIN Project n. 2022TS4Y3N - EXPAND: scalable algorithms for EXPloratory Analyses of heterogeneous and dynamic Networked Data, funded by the Italian Ministry of University and Research (MUR), and by the project BRAINTEASER (Bringing Artificial Intelligence home for a better care of amyotrophic lateral sclerosis and multiple sclerosis), funded by European Union's Horizon 2020 (grant agreement No. GA101017598).
\end{acks}

\bibliographystyle{ACM-Reference-Format}
\balance
\bibliography{bibliography}

\ifextversion
\clearpage
\newpage
\fi
\appendix


\section{Appendix}

In this Appendix we provide proofs and additional experimental results that could not fit in the main paper due to space constraints. 

\subsection{Proofs of \Cref{sec:VCrulelists}}

\begin{proof}[Proof of \Cref{thm:vcupperboundgen}]
Given a dataset $\dataset$ with $d$ features, we create a new dataset $\dataset^z$ 
built as follows.
Let $C = \{ \qtm{x_1 = 1} , \qtm{x_2 = 1} , \dots , \qtm{x_d = 1}  \}$ be the set of all $d$ possible conditions on the $d$ binary features of $\dataset_s$;
for any non-empty subset $A \subseteq C$ with $|A| \leq z$, we add to $\dataset^z$ the binary feature $x^\prime_A$ that is equal to $1$ for all training instances of $\dataset$ such that $\bigwedge_{c \in A} c$ is true.
Equivalently, the feature values of $x^\prime_A$ are obtained from the logical AND of the evaluations of the conditions in $A$. 
It follows that the total number of features of $\dataset^z$ is $d^\prime = \sum_{i=1}^z \binom{d}{i}$.

We now observe that the set $\mathcal{P}(\tilde{\ruleset}_k^z, \dataset)$  of projections of rule lists in $\ruleset_{k}^{z}$ on $\dataset$ is contained in the set $\mathcal{P}(\tilde{\ruleset}_k^1,\dataset^{z})$  of projections of rules in $\ruleset_{k}^{1}$ on $\dataset^z$, 
since we can replace each rule with a conjunctions with $t$ terms, with $1 \leq t \leq z$, in any rule list $\in \ruleset_{k}^{z}$ on $\dataset$ 
by a rule with a single condition on one of the features of $\dataset^z$, obtaining an equivalent rule list from $\ruleset_{k}^{1}$. 
Therefore, from \Cref{thm:vcupperbound} applied to $\ruleset_{k}^{1}$ over a dataset with $d^\prime$ features, and the fact $d^\prime = \sum_{i=1}^z \binom{d}{i} \leq (\frac{ed}{z})^z$, we obtain
\begin{align*}
VC(\ruleset_{k}^{z}) 
\leq \left\lfloor k \log_{2} \pars{ 2 d^\prime } + 2 \right\rfloor 
= \left\lfloor k \log_{2} \pars{ 2 \sum_{i=1}^z \binom{d}{i} } + 2 \right\rfloor \\
\leq \left\lfloor k \log_{2} \pars{ 2 \pars{ \frac{ed}{z} }^z } + 2 \right\rfloor 
\leq \left\lfloor k z \log_{2} \pars{ \frac{2 e d}{z} } + 2 \right\rfloor ,
\end{align*}
and the statement follows. 
\end{proof}
\ifextversion
\begin{proof}[Proof of \Cref{thm:vclowerbound}]
To show the lower bound, we build a dataset that is shattered by the rangeset defined by $\ruleset_{k}^{1}$.
Let the integers $a,b$ such that $a \geq 1$ and $b=2^a-1$, and define the binary matrix $C_a \in \{0,1\}^{a \times b}$ such that each column is a distinct binary vector with at least one element equal to $1$, i.e., that the set of columns has cardinality $2^a-1$. 
Equivalently, the columns of $C_a$ represent all the distinct non-empty subsets of a set of $a$ items. 
For instance, for $a = 3$ we have $b=7$ and $C_3$ as 
\begin{align*}
C_3 = \begin{bmatrix} 
1 & 0 & 0 & 1 & 1 & 0 & 1\\
0 & 1 & 0 & 1 & 0 & 1 & 1\\
0 & 0 & 1 & 0 & 1 & 1 & 1
\end{bmatrix} . 
\end{align*}
For $a , k \geq 1$, define the matrices $A_i = C_a$ for all $i \in [1,k]$, and let the dataset $\dataset(a,k)$ with $ak$ instances and $d=bk$ features be defined as the following $ak \times bk$ binary matrix:
\begin{align*}
\dataset(a,k) = \begin{bmatrix} 
A_1 & \textbf{0} & \textbf{0} & \dots  & \textbf{0} \\
\textbf{0} & A_2 & \textbf{0} & \dots & \textbf{0} \\
\textbf{0} & \textbf{0} & A_3 & \dots & \dots \\
\dots & \dots & \dots & \dots & \textbf{0} \\
\textbf{0} & \textbf{0} & \dots & \textbf{0} & A_k 
\end{bmatrix} . 
\end{align*}
We observe that the range set $\tilde{\ruleset}_{k}^{1}$ shatters $\dataset(a,k)$.
In fact, let any binary vector $y \in \{0,1\}^{ak}$, and, for $i \in [1,k]$, define $y^i \in \{0,1\}^{ak}$ as a copy of $y$ where all the elements with indices $\leq (i-1)a$ or $\geq ia+1$ are set to $0$. 
Therefore, for all the indices $j \in [ (i-1)a+1 , ia ]$ it holds $y^i_j = y_j$.
Let $s_j$ be the $j$-th row of $\dataset(a,k)$, for all $j \in [1,ak]$. 
Define the function $h : \{0,1\}^{ak} \rightarrow [1,bk]$ such that $h(y^i)$ is the index of the column 
of $\dataset(a,k)$ that is equal to $y^i$. 
Note that such column corresponds to one of the columns of the matrix $A_i$. 

We prove that there is at least one rule list $R \in \ruleset_{k}^{1}$ whose predictions $[P(R,s_1) , P(R,s_2) , \dots , P(R,s_{ak})]$ form a binary vector equal to $y$. 
We define $R$ as follows: 
let $V \subseteq [1,k]$ be the subset of indices such that 
$i \in V$ if and only if $y^i$ contains at least one entry equal to $1$.
We define $R$ as a rule list of length $u=|V|$, composed by the list $[ r_i : i \in V]$ and the default rule 
$\defrule = ( c_{\defsymbol} , 0 )$.
Each rule $r_i$, corresponding to the index $i \in V$, is $r_i = (c_i , 1)$ where the condition $c_i$ is defined as follows: we take the condition \qt{$h(y^i) = 1$}, i.e., using the column of the dataset $\dataset(a,k)$ equal to $y^i$. 
We observe that, for any binary vector $y\in \{0,1\}^{ak}$, 
it holds
\begin{align*}
[P(R,s_1) , P(R,s_2) , \dots , P(R,s_{ak})] = \bigvee_{i \in [1,k]} h(y^i) = \bigvee_{i \in [1,k]} y^i = y,
\end{align*}
where $\bigvee$ denotes the element-wise logical \texttt{OR} between binary vectors. 
We conclude that 
the rule list $R$ defined above perfectly classifies all elements of the dataset $\dataset(a,k)$ when labelled by $y$. 

Therefore, there exist a dataset with $ak$ instances and $d=bk$ features that is shattered.
First, note that 
\begin{align*}
d = bk = (2^a-1) k \implies a = \log_2 \pars{ \frac{d+k}{k} },
\end{align*}
consequently, it holds
\begin{align*}
VC(\ruleset_{k}^{1}) \geq ak = k \log_2 \pars{ \frac{d+k}{k} }. 
\end{align*}
\end{proof}

We now prove \Cref{thm:vclowerboundgen}. 
We make use of the following result, that provides a lower bound to the VC-dimension of the class of $k$-term monotone $z$-DNF boolean functions.

\begin{lemma}[Lemma 6 of \cite{littlestone1988learning}]
\label{thm:lowerboundvcdnf}
For $1 \leq z \leq d$ and $1 \leq k \leq \binom{d}{z}$, let $D_k^z$ be the class of functions expressible as $k$-term monotone $z$-DNF formulas over $d$ binary features and let $a$ be any integer, $z \leq a \leq d$ such that $\binom{a}{z} \geq k$. Then 
$VC(D_k^z) \geq \left\lfloor k z \log_{2} \pars{\frac{d}{a }} \right\rfloor$.
\end{lemma}

\begin{proof}[Proof of \Cref{thm:vclowerboundgen}]
For $z,k \geq 1$, denote with $D_k^z$ the class of $k$-term monotone $z$-DNF over the $d$ features of the dataset.
We note that, for any formula $F \in D_k^z$, we can build an equivalent rule list $R$ that, for any $s \in \{0,1\}^{d}$, provides the same predictions. 
To do so, let $F = \bigvee_{i=1}^{h} c_{i}$ where $h \leq k$ and each $c_{i}$ is a conjunction with at most $z$ monotone terms.
It is easy to observe that the rule list $R = [ (c_{1} , 1) , (c_{2} , 1) , \dots , (c_{h} , 1) , (c_{\defsymbol} , 0)  ]$ is equivalent to $F$. 
This mapping implies that $D_k^z \subseteq \ruleset_{k}^{z}$, therefore $VC(D_k^z) \leq VC(\ruleset_{k}^{z})$. 
From \Cref{thm:lowerboundvcdnf}, we observe that the VC-dimension of $D_k^z$ is
\begin{align*}
VC(D_k^z) \geq \left\lfloor k z \log_{2} \pars{\frac{d}{a }} \right\rfloor
\end{align*}
for some $z \leq a \leq d$ with $\binom{a}{z} \geq k$. It holds
\begin{align*}
\binom{a}{z} \geq \pars{ \frac{a}{z} }^z \geq k \implies a \geq \sqrt[z]{k} z ,
\end{align*}
obtaining that 
\begin{align*}
VC(D_{k}^{z}) &\geq \left\lfloor k z \log_{2} \pars{\frac{d}{z \sqrt[z]{k} }} \right\rfloor.
\end{align*}
The statement follows from $VC(D_k^z) \leq VC(\ruleset_{k}^{z})$.
\end{proof}

\fi

\subsection{Proofs of \Cref{sec:guarantees}}
To prove our results we use the following Chernoff bounds (see Theorem 4.4 and 4.5 of \cite{mitzenmacher2017probability}). 
\begin{theorem}
\label{thm:chernoffbounds}
Let $X_1 , \dots X_m$ be independent Poisson trials such that $\Pr(X_i = 1) = p_i$.
Let $X = \sum_{i=1}^n X_i$ and $\mu = \E[X]$. 
Then the following Chernoff bounds hold for any $0 < \gamma < 1$:
\begin{align*}
\Pr\pars{ X \geq (1+\gamma)\mu } \leq \exp ( -\mu\gamma^2/3 ) , \\
\Pr\pars{ X \leq (1-\gamma)\mu } \leq \exp ( -\mu\gamma^2/2 ) . 
\end{align*}
\end{theorem}

\begin{proof}[Proof of \Cref{thm:approxbounds}]
We prove the first set of inequalities.
Let $R$ be an arbitrary rule list from $\ruleset_{k}^{z}$,
and define the functions 
\begin{align}
g(R,\sample) &= \ell(R,\sample) - \alpha|R| , \\
g(R,\dataset) &= \ell(R,\dataset) - \alpha|R| . 
\end{align}
Note that $g(R,\dataset)$ and $g(R,\sample)$ are the non-regularized variants of the loss functions $\ell(R,\dataset)$ and $\ell(R,\sample)$. 
It easy to show that $\E_{\sample}[ g(R,\sample) ] = g(R,\dataset) \leq \ell(R,\dataset)$, and that 
$g(R,\sample)$ is an average of $m$ binary random variables with expectation $g(R,\dataset)$. 
From an application of the Chernoff bound (\Cref{thm:chernoffbounds}) to the random variable $Z = mg(R,\sample)$ with $\E_\sample[Z] = mg(R,\dataset)$, we have that, for any $0 < \gamma < 1$,  
\begin{align}
\Pr \pars{ Z \leq (1-\gamma)\E_\sample[Z] } \leq \exp\pars{ -\E_\sample[Z] \gamma^2/2 }. \label{eq:probdevall}
\end{align}
Fixing $0 < \delta^\prime < 1$, imposing the r.h.s. of \eqref{eq:probdevall} to be $\leq \delta^\prime$, and solving for $\gamma$, gives
\begin{align*}
\Pr \pars{ g(R,\sample) + \sqrt{ \frac{2 g(R,\dataset) \ln(1/\delta^\prime)}{m} } \leq g(R,\dataset) } \leq \delta^\prime . 
\end{align*}
Define the events 
\begin{align*}
E_R = \qtm{ g(R,\sample) + \sqrt{ \frac{2 g(R,\dataset) \ln(1/\delta^\prime)}{m} } \leq g(R,\dataset) }, \forall R \in \ruleset_{k}^{z},
\end{align*}
and the event $E = \qtm{ \exists R : E_R \text{ is true}}$. 
Note that $\Pr(E_R)\leq\delta^{\prime}$. 
We want to prove that $\Pr(E)\leq \delta/2$. 
First, we observe that 
\begin{align*}
\Pr\pars{ E } = \Pr \biggl( \bigcup_{R \in \ruleset_{k}^{z}} E_R \biggr) .
\end{align*}
Denote with $R_{1}$ and $R_{2}$ two rule lists $\in \ruleset_{k}^{z}$ such that the projections $X(R_{i} , \dataset)$ of $R_{i}$ on $\dataset$ are equal: 
$X(R_{1} , \dataset) = X(R_{2} , \dataset)$. 
Consequently, it holds $g(R_{1},\dataset) = g(R_{2},\dataset)$.
Moreover, for all possible samples $\sample$, 
it holds $X(R_{1} , \sample) = X(R_{2} , \sample)$ 
and $g(R_{1},\sample) = g(R_{2},\sample)$. 
The observations above imply that the events $E_{R_1}$ and $E_{R_2}$ are identical. 

We now define the set $\mC( \ruleset_{k}^{z} , \dataset ) \subseteq \ruleset_{k}^{z}$ as a cover of the rangeset $\tilde{\ruleset}_{k}^{z}$ as follows:
fix an arbitrary, deterministic total order over all elements of $\ruleset_{k}^{z}$; 
then, for every distinct projection $H \in \tilde{\ruleset}_{k}^{z}$, 
let $\{ R : X(R,\dataset) = H , R \in \ruleset_{k}^{z} \}$ be the set of rule lists with projection equal to $H$. 
From this set, 
we pick the minimum element in the total order over $\ruleset_{k}^{z}$ 
and we include it in
$\mC( \ruleset_{k}^{z} , \dataset )$. 
Therefore, $\mC( \ruleset_{k}^{z} , \dataset )$ contains a unique rule list for each distinct projection of the rangeset $\tilde{\ruleset}_{k}^{z}$ over the dataset $\dataset$. 
This implies that, for any $R \in \ruleset_{k}^{z}$, there is an unique rule list $R^{\prime} \in \mC( \ruleset_{k}^{z} , \dataset )$ such that 
$X(R , \dataset) = X(R^{\prime} , \dataset)$
and that the events $E_{R}$ and $E_{R^{\prime}}$ are identical. 
It follows  
\begin{align*}
\Pr \biggl( \bigcup_{R \in \ruleset_{k}^{z}} E_R \biggr) 
= \Pr \biggl( \bigcup_{R \in \mC( \ruleset_{k}^{z} , \dataset )} E_R \biggr) ,
\end{align*}
as we can replace the union over all $R \in \ruleset_{k}^{z}$ with the union over the cover $\mC( \ruleset_{k}^{z} , \dataset )$ containing rule lists with unique projections on $\dataset$.
From the definitions of  $\Lambda(\ruleset_{k}^{z} , n)$ and $\mC( \ruleset_{k}^{z} , \dataset )$, it holds
\begin{align*}
|\mC( \ruleset_{k}^{z} , \dataset )| = |\tilde{\ruleset}_{k}^{z}(\dataset)| \leq \Lambda(\ruleset_{k}^{z} , n) .
\end{align*}
From an union bound, we have
\begin{align*}
\Pr(E) = \Pr \biggl( \bigcup_{R \in \mC( \ruleset_{k}^{z} , \dataset )} E_R \biggr) 
\leq \sum_{R \in \mC( \ruleset_{k}^{z} , \dataset )} \Pr\pars{ E_R } \leq \Lambda(\ruleset_{k}^{z} , n) \delta^\prime .
\end{align*}
Choosing $\delta^\prime = \delta/(2\Lambda(\ruleset_{k}^{z} , n))$, we obtain that
\begin{align}
g(R,\dataset) &\leq g(R,\sample) \!+\! \sqrt{\frac{2 g(R,\dataset) \! \pars{ \omega \! + \! \logtdelta } }{ m }}  \label{eq:boundfixedpoint}
\end{align}
holds for all $R \in \ruleset_{k}^{z}$ with probability $\geq 1-\delta/2$, 
since $\ln(\delta^\prime) \leq \omega + \logtdelta$ (following analogous derivations for the proof of \Cref{thm:vcupperbound}). 
The first set of inequalities is obtained from \eqref{eq:boundfixedpoint} after 
adding $\alpha |R|$ on both sides,
from the fact that $g(R,\dataset) \leq \ell(R,\dataset)$, 
after solving the quadratic inequality $y \leq u + \sqrt{vy}$ w.r.t. $y$, 
and after straightforward computations. 

We now prove the last inequality and the statement.
Let $R^\star$ be an arbitrary rule list with $\ell(R^\star, \dataset) = \min_{R \in \ruleset_{k}^{z}} \ell(R , \dataset)$. 
Note that $R^\star$ is fixed and independent of the choice of $\sample$. 
We apply the Chernoff bound to the random variable $Z = mg(R^\star,\sample)$ with $\E_\sample[Z] = m g(R^\star,\dataset)$, obtaining for any $0 < \gamma < 1$   
\begin{align}
\Pr \pars{ Z \geq (1+\gamma)\E_\sample[Z] } \leq \exp\pars{ -\E_\sample[Z] \gamma^2/3 }. \label{eq:probdevbest}
\end{align}
Setting the r.h.s. of \eqref{eq:probdevbest} $\leq \delta/2$, and solving for $\gamma$ using $\alpha \geq 0$, gives
\begin{align*}
\Pr \pars{ g(R^\star,\sample) \geq g(R^\star,\dataset) + \sqrt{ \frac{3 g(R^\star,\dataset) \ln(2/\delta)}{m} } } \leq \delta/2 , 
\end{align*}
and, equivalently,
\begin{align*}
g(R^\star,\sample) \leq g(R^\star,\dataset) + \sqrt{ \frac{3 g(R^\star,\dataset) \ln(2/\delta)}{m}  }
\end{align*}
with probability $\geq 1-\delta/2$. 
We obtain the last inequality of the statement after adding $\alpha |R^\star|$ to both sides, 
and from the fact $g(R^\star,\dataset) \leq \ell(R^\star,\dataset)$. 
From an union bound, all inequalities of the theorem are simultaneously valid with probability $\geq 1-\delta$, obtaining the statement. 
\end{proof}

\begin{figure*}[ht]
\begin{subfigure}{.7\textwidth}
  \centering
  \includegraphics[width=\textwidth]{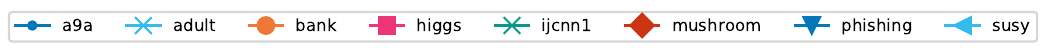}
\end{subfigure} \\
\begin{subfigure}{.32\textwidth}
  \centering
  \includegraphics[width=\textwidth]{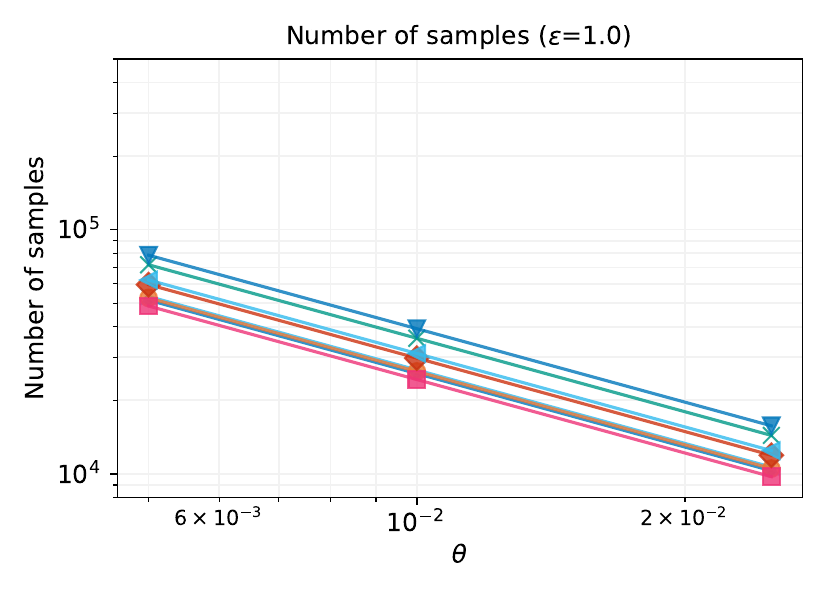}
  \caption{}
\end{subfigure}
\begin{subfigure}{.32\textwidth}
  \centering
  \includegraphics[width=\textwidth]{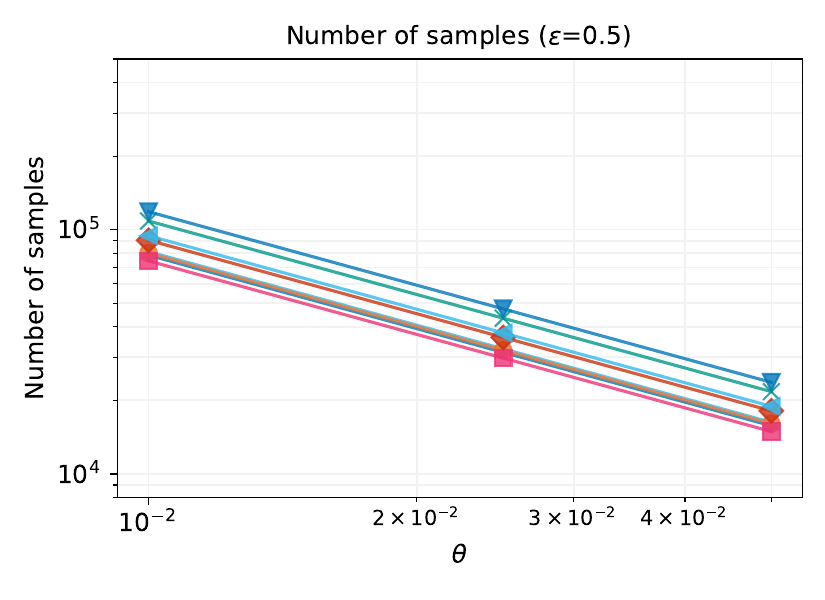}
  \caption{}
\end{subfigure}
\begin{subfigure}{.32\textwidth}
  \centering
  \includegraphics[width=\textwidth]{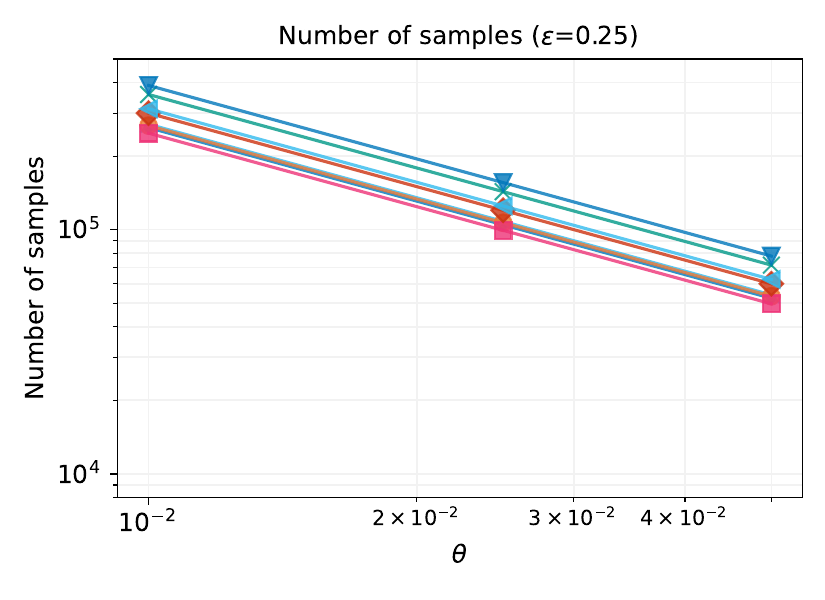}
  \caption{}
\end{subfigure}
\caption{ 
Number of samples $\hat{m}(\ruleset_{k}^{z} , \dataset)$ used by \algname\ varying $\varepsilon$ and $\theta$ for all datasets. $k$ is set as in \Cref{tab:datasets} and $z=1$. 
}
\label{fig:samplesizes}
\Description{The figures show the number of samples considered by SamRuLe for different values of the parameters.}
\end{figure*}

\begin{figure*}[ht]
\begin{subfigure}{.55\textwidth}
  \centering
  \includegraphics[width=\textwidth]{./figures/plot-legend-time-cropped.pdf}
\end{subfigure} \\
\begin{subfigure}{.24\textwidth}
  \centering
  \includegraphics[width=\textwidth]{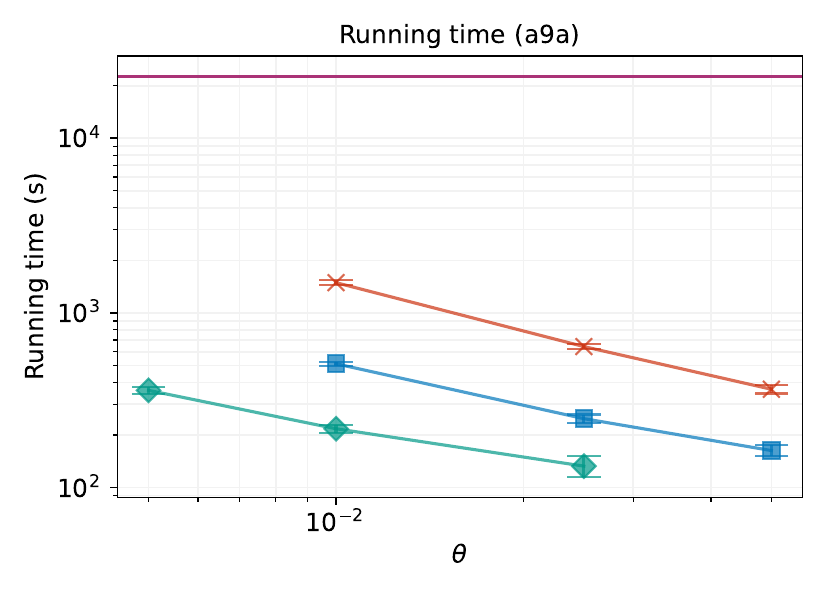}
\end{subfigure}
\begin{subfigure}{.24\textwidth}
  \centering
  \includegraphics[width=\textwidth]{./figures/running-time-adult.pdf}
\end{subfigure}
\begin{subfigure}{.24\textwidth}
  \centering
  \includegraphics[width=\textwidth]{./figures/running-time-bank.pdf}
\end{subfigure}
\begin{subfigure}{.24\textwidth}
  \centering
  \includegraphics[width=\textwidth]{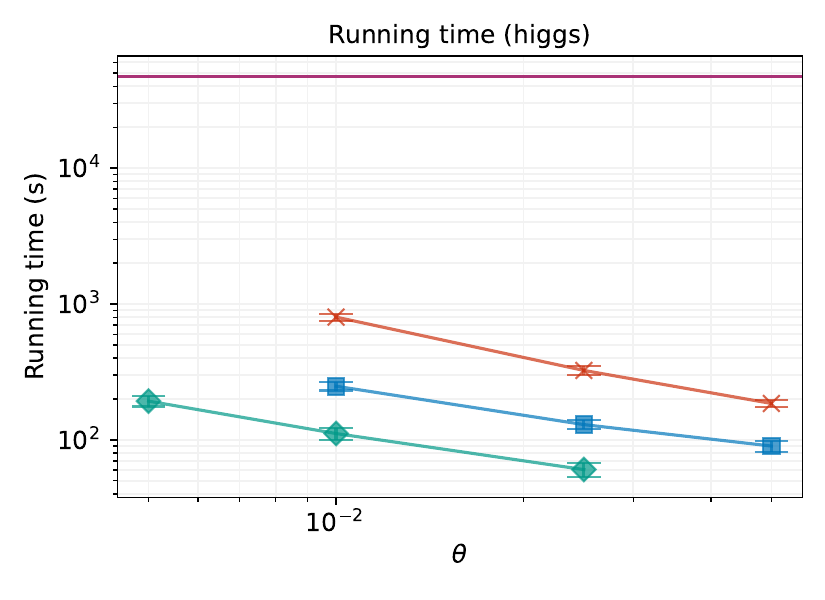}
\end{subfigure}
\begin{subfigure}{.24\textwidth}
  \centering
  \includegraphics[width=\textwidth]{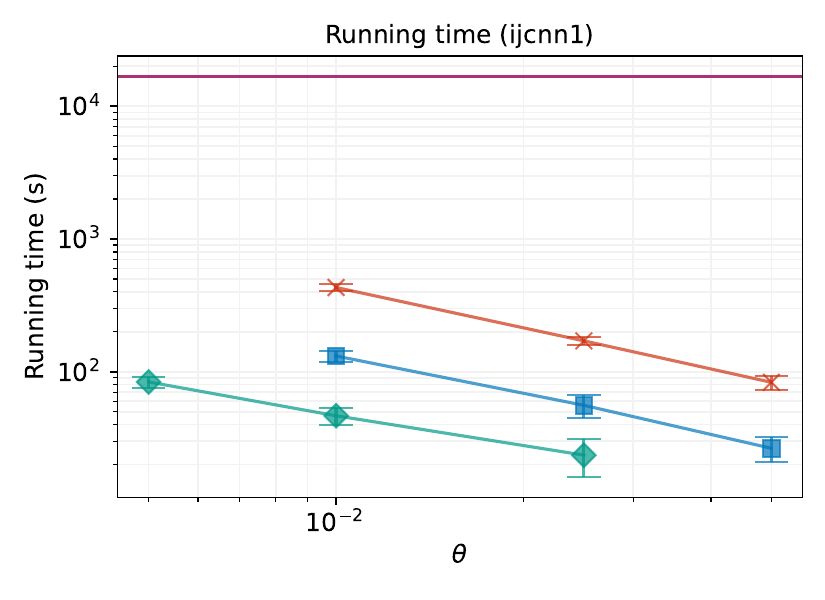}
\end{subfigure}
\begin{subfigure}{.24\textwidth}
  \centering
  \includegraphics[width=\textwidth]{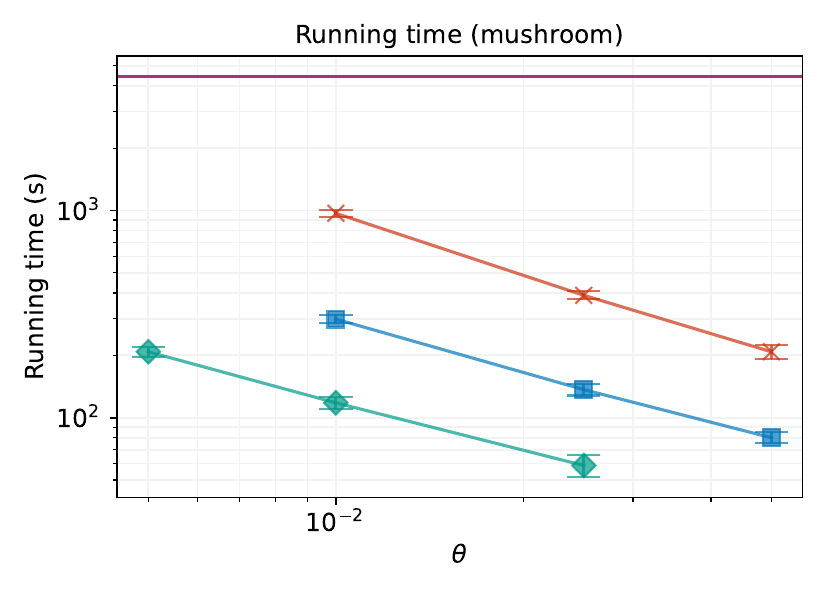}
\end{subfigure}
\begin{subfigure}{.24\textwidth}
  \centering
  \includegraphics[width=\textwidth]{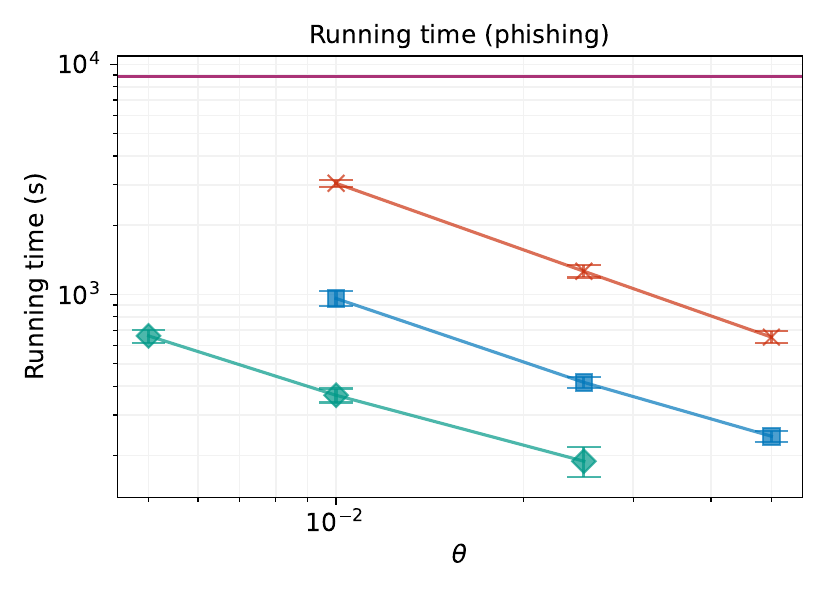}
\end{subfigure}
\begin{subfigure}{.24\textwidth}
  \centering
  \includegraphics[width=\textwidth]{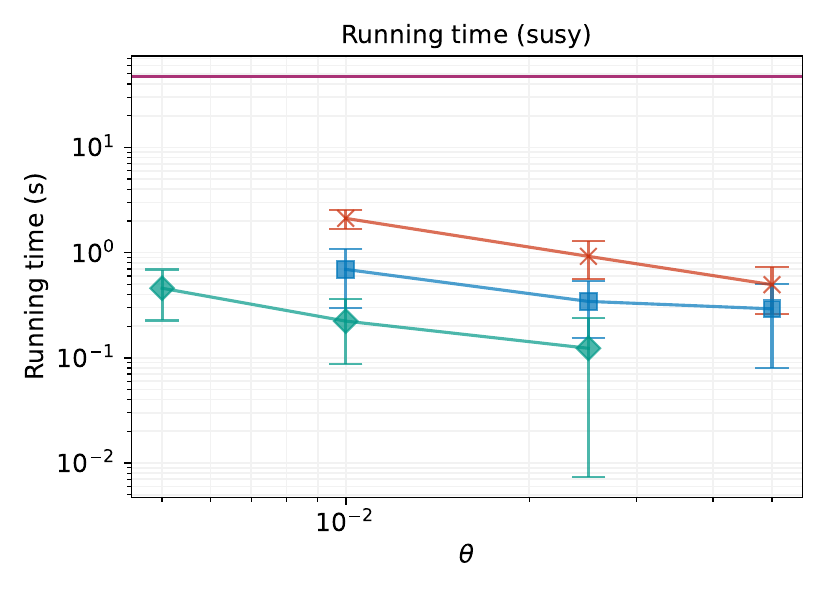}
\end{subfigure}
\caption{ 
Running times of \algname\ and CORELS for different values of $\varepsilon$ and $\theta$ on all datasets. 
}
\label{fig:runningtimes}
\Description{The figures show a performance comparison between SamRuLe and CORELS on all datasets, for different values of the parameters. SamRuLe needs a fraction of the time of CORELS, since it considers a small random sample of the large dataset.}
\end{figure*}

\begin{figure*}[ht]
\ifextversion
\begin{subfigure}{.55\textwidth}
  \centering
  \includegraphics[width=\textwidth]{./figures/plot-legend-time-cropped.pdf}
\end{subfigure} \\
\fi
\begin{subfigure}{.24\textwidth}
  \centering
  \includegraphics[width=\textwidth]{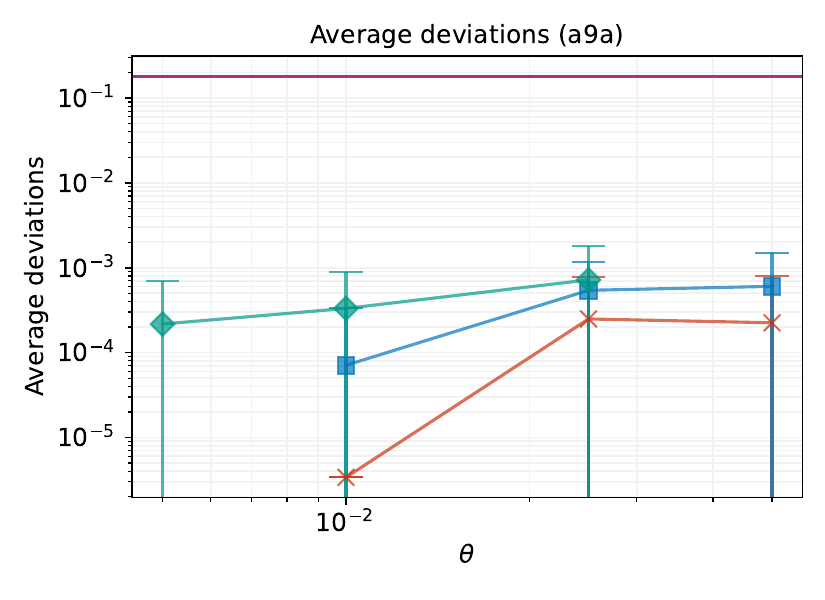}
\end{subfigure}
\begin{subfigure}{.24\textwidth}
  \centering
  \includegraphics[width=\textwidth]{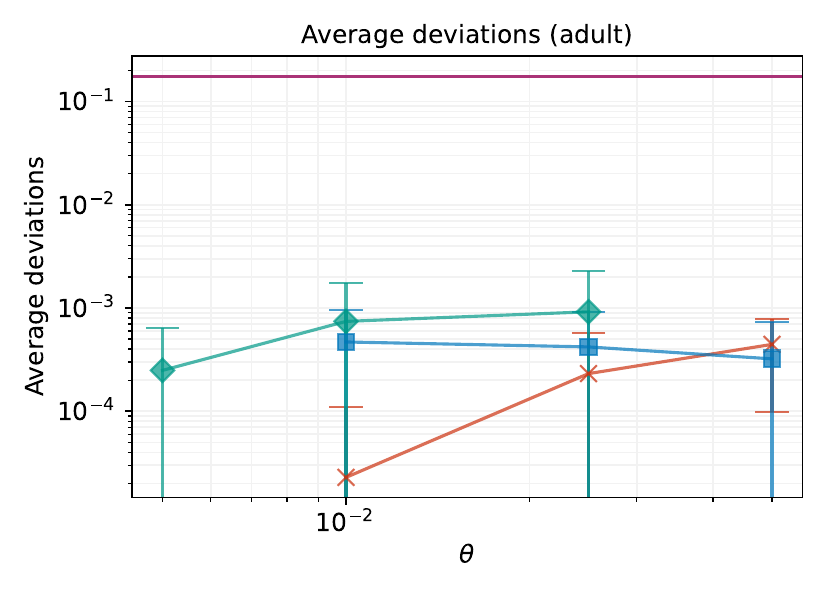}
\end{subfigure}
\begin{subfigure}{.24\textwidth}
  \centering
  \includegraphics[width=\textwidth]{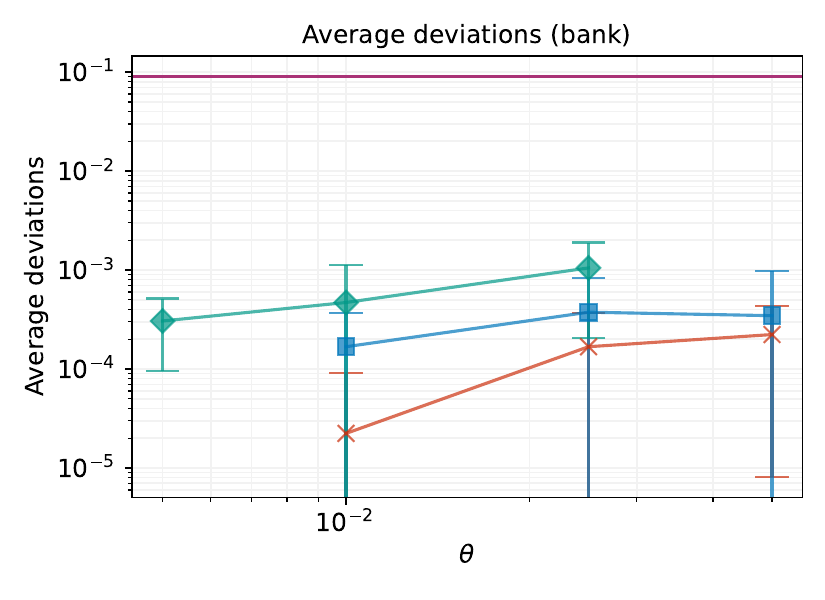}
\end{subfigure}
\begin{subfigure}{.24\textwidth}
  \centering
  \includegraphics[width=\textwidth]{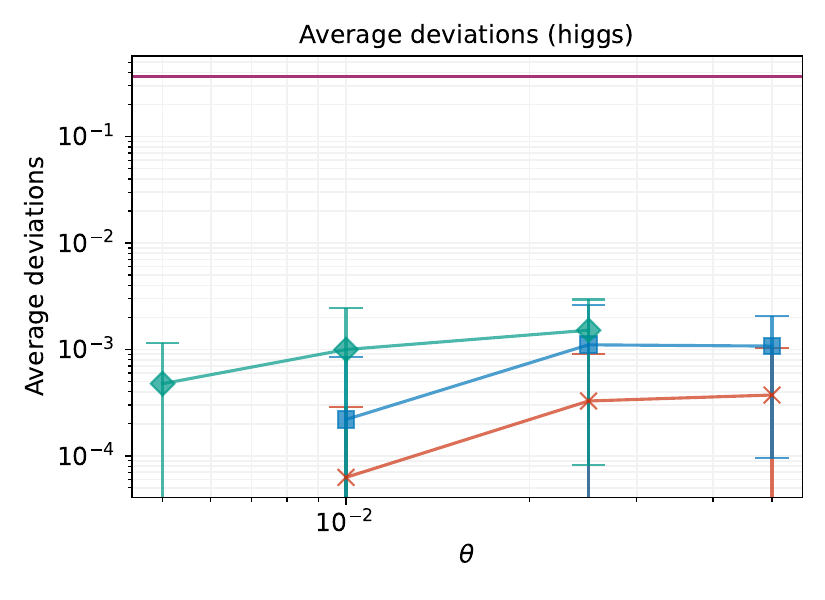}
\end{subfigure}
\begin{subfigure}{.24\textwidth}
  \centering
  \includegraphics[width=\textwidth]{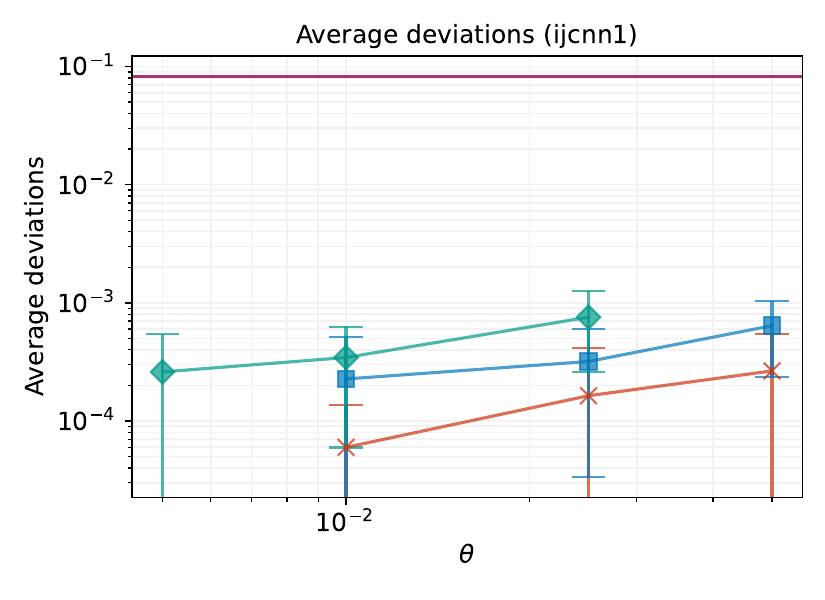}
\end{subfigure}
\begin{subfigure}{.24\textwidth}
  \centering
  \includegraphics[width=\textwidth]{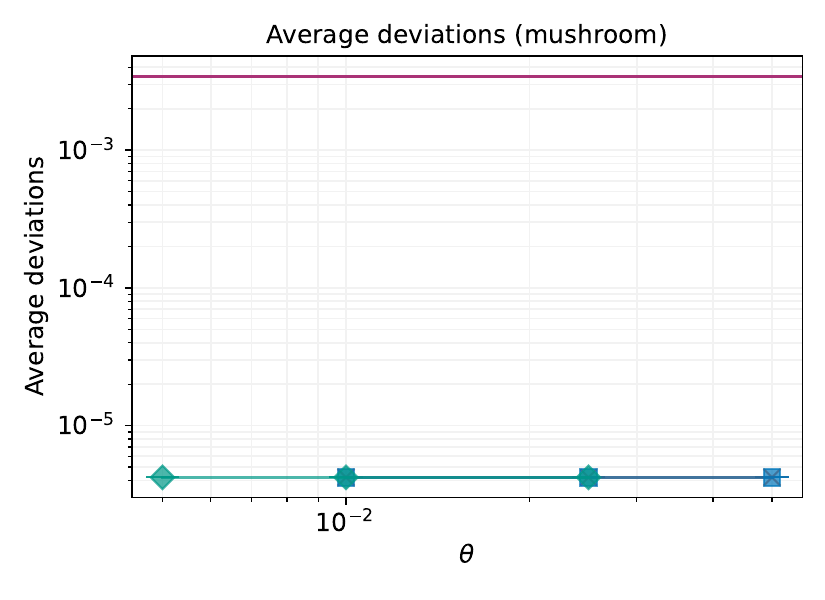}
\end{subfigure}
\begin{subfigure}{.24\textwidth}
  \centering
  \includegraphics[width=\textwidth]{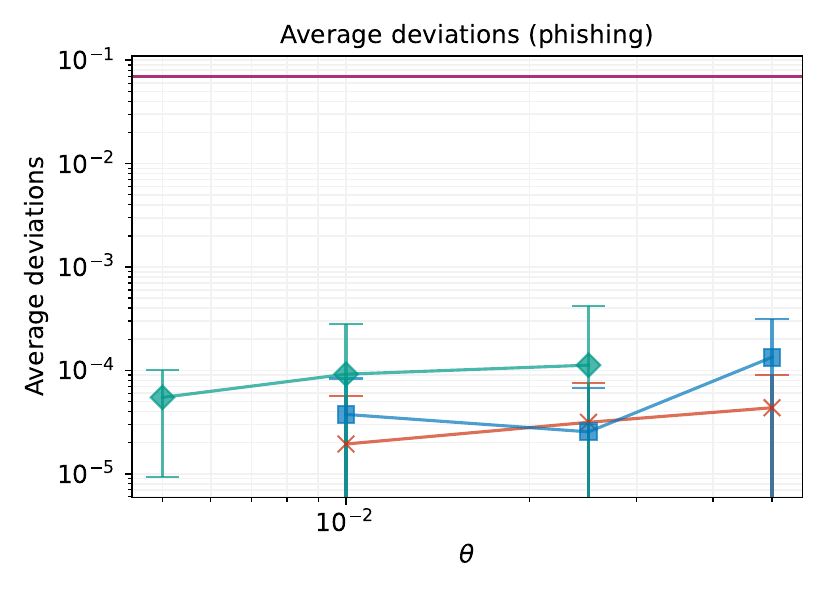}
\end{subfigure}
\begin{subfigure}{.24\textwidth}
  \centering
  \includegraphics[width=\textwidth]{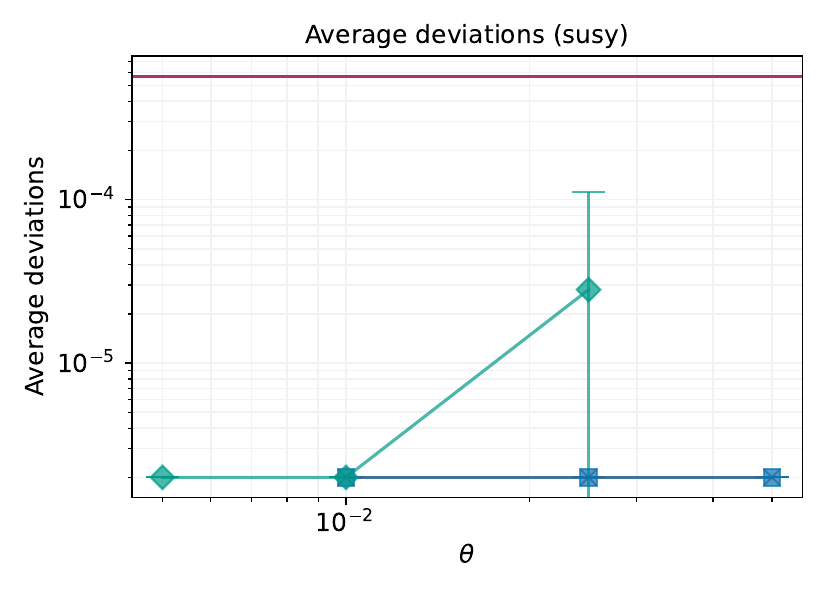}
\end{subfigure}
\caption{ 
Average deviations $| \ell(\tilde{R} , \dataset) - \ell(R^\star , \dataset) |$ over all runs for different values of $\varepsilon$ and $\theta$ and for all datasets. 
The purple horizontal line is drawn at $y = \ell(R^\star , \dataset)$ (the loss of the optimal rule list $R^\star$ found by CORELS). 
}
\label{fig:avgdevs}
\Description{The figures show an accuracy comparison between SamRuLe and CORELS on all datasets, for different values of the parameters. SamRuLe very accurate rule list, with loss very close to the optimal.}
\end{figure*}

\ifextversion
\begin{figure*}[ht]
\begin{subfigure}{.55\textwidth}
  \centering
  \includegraphics[width=\textwidth]{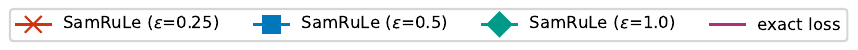}
\end{subfigure} \\
\begin{subfigure}{.24\textwidth}
  \centering
  \includegraphics[width=\textwidth]{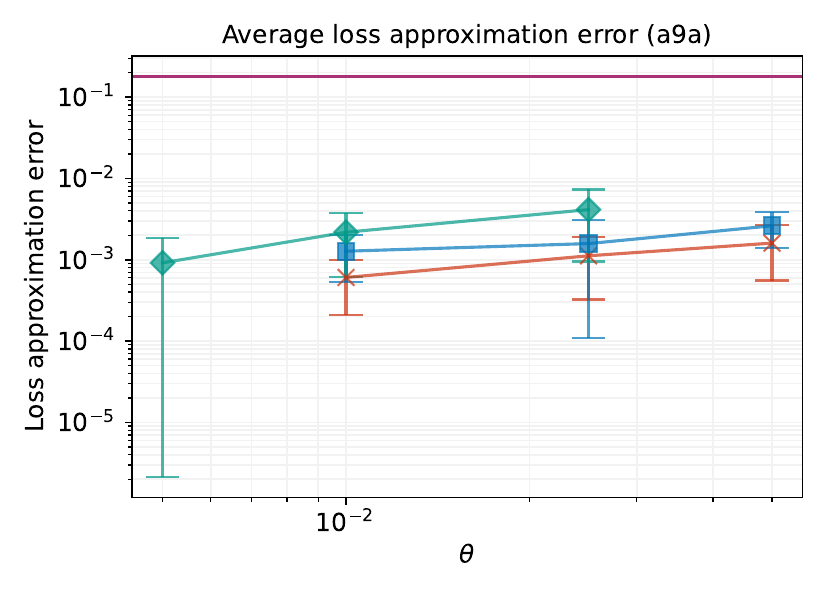}
\end{subfigure}
\begin{subfigure}{.24\textwidth}
  \centering
  \includegraphics[width=\textwidth]{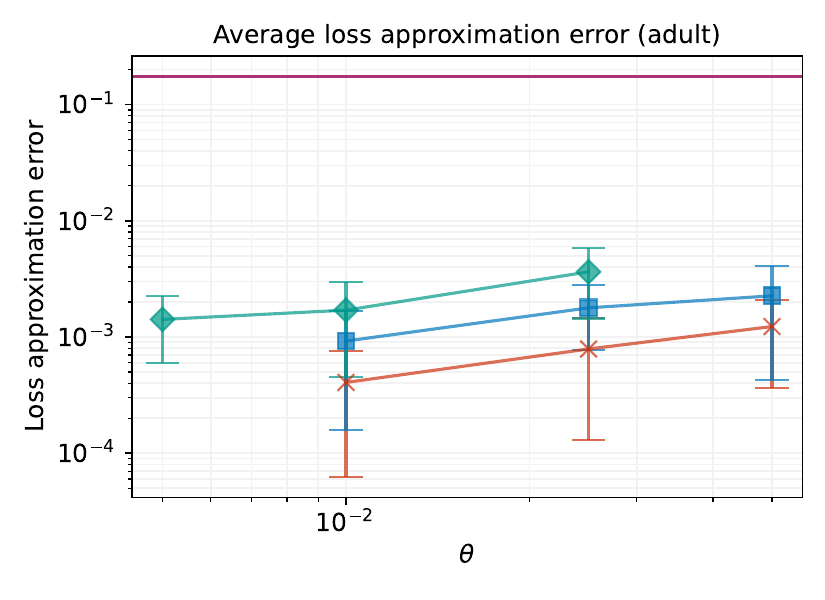}
\end{subfigure}
\begin{subfigure}{.24\textwidth}
  \centering
  \includegraphics[width=\textwidth]{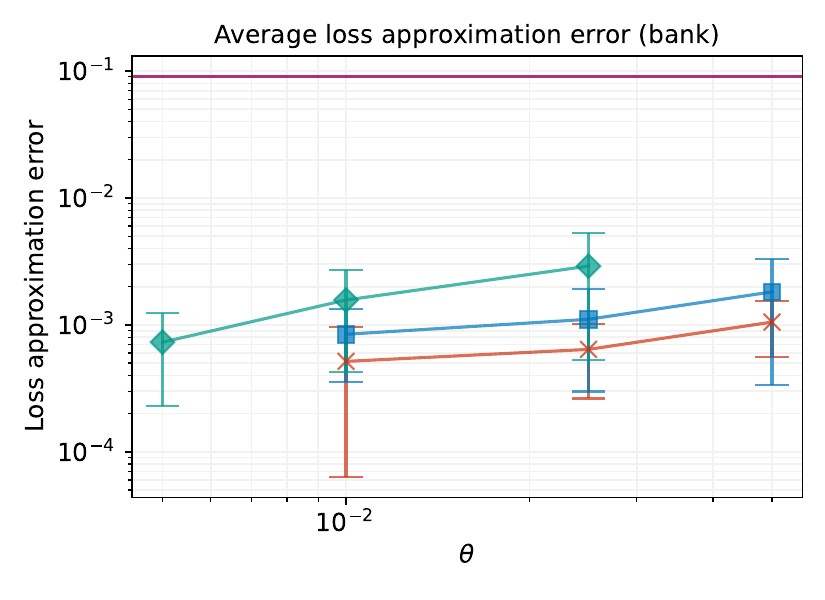}
\end{subfigure}
\begin{subfigure}{.24\textwidth}
  \centering
  \includegraphics[width=\textwidth]{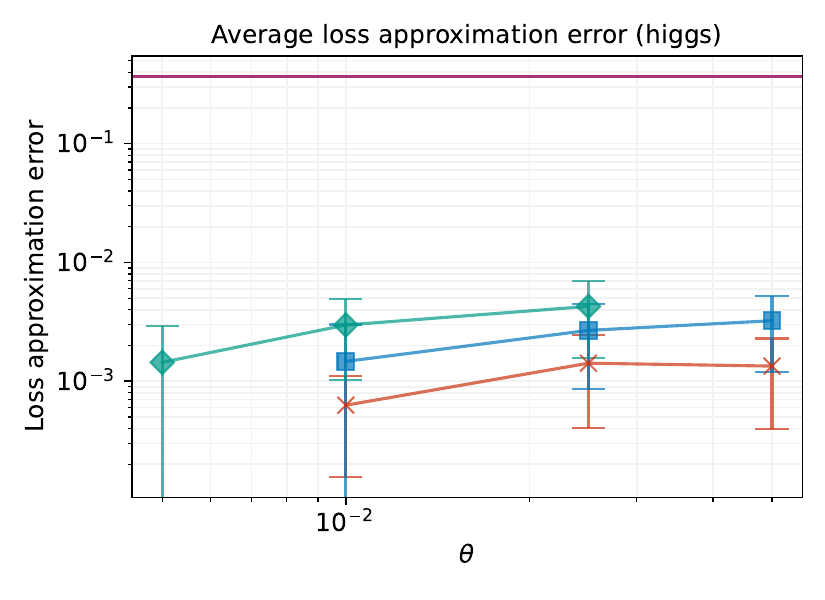}
\end{subfigure}
\begin{subfigure}{.24\textwidth}
  \centering
  \includegraphics[width=\textwidth]{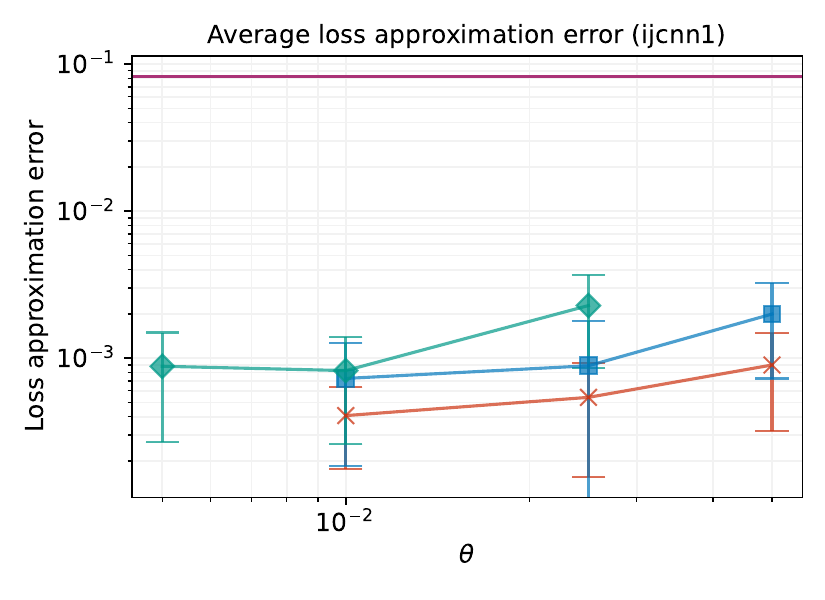}
\end{subfigure}
\begin{subfigure}{.24\textwidth}
  \centering
  \includegraphics[width=\textwidth]{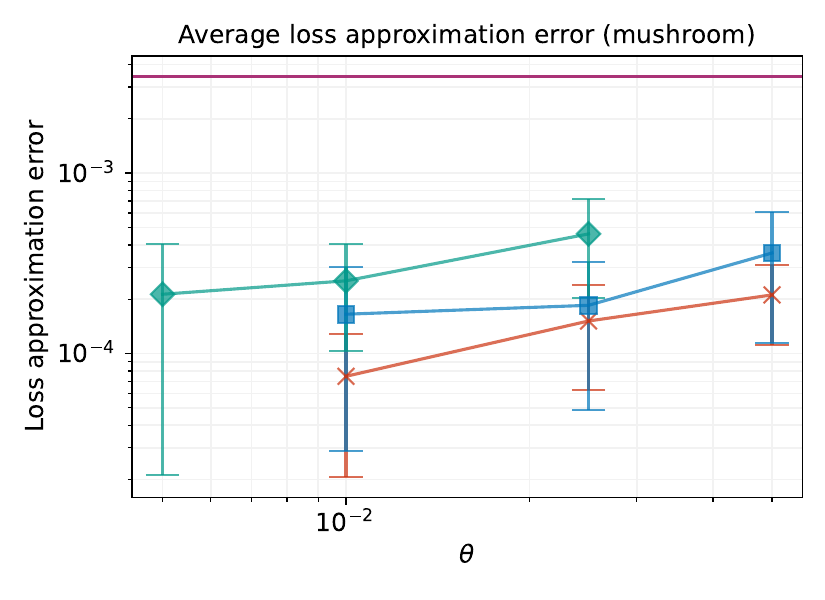}
\end{subfigure}
\begin{subfigure}{.24\textwidth}
  \centering
  \includegraphics[width=\textwidth]{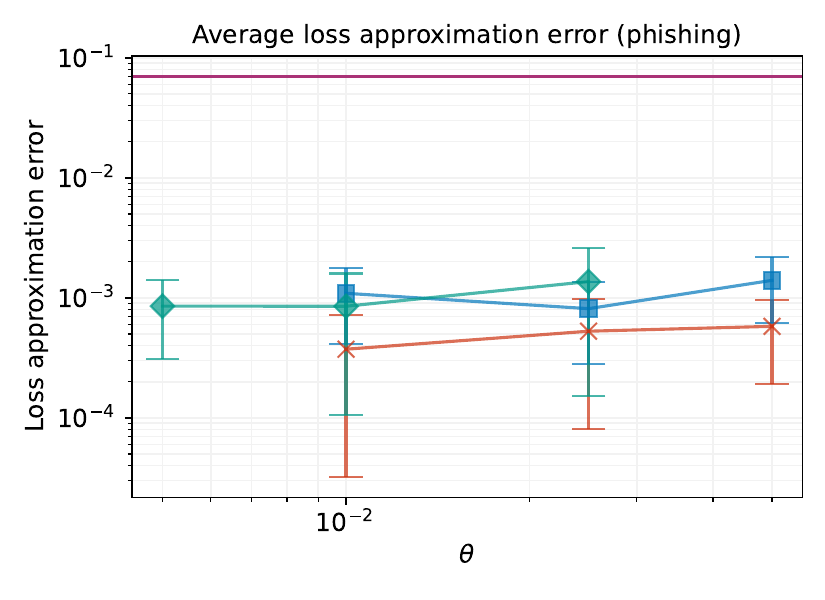}
\end{subfigure}
\begin{subfigure}{.24\textwidth}
  \centering
  \includegraphics[width=\textwidth]{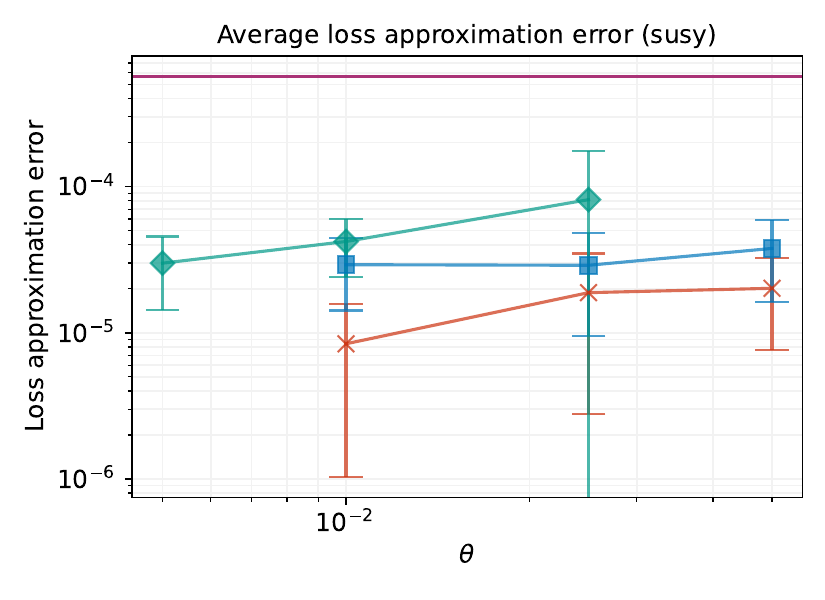}
\end{subfigure}
\caption{ 
Average loss approximation error $| \ell(\tilde{R} , \dataset) - \ell(\tilde{R} , \sample) |$ over all runs for different values of $\varepsilon$ and $\theta$ and for all datasets. 
The purple horizontal line is drawn at $y = \ell(\tilde{R} , \dataset)$ (the loss on the dataset for the rule list $\tilde{R}$ found by \algname). 
}
\label{fig:avgdevssampleloss}
\Description{The figures show an accuracy comparison between SamRuLe and CORELS on all datasets, for different values of the parameters. SamRuLe very accurate rule list, with loss very close to the optimal.}
\end{figure*}

\begin{figure*}[ht]
\begin{subfigure}{.4\textwidth}
  \centering
  \includegraphics[width=\textwidth]{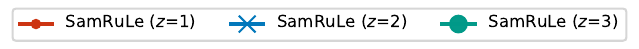}
\end{subfigure} \\
\begin{subfigure}{.32\textwidth}
  \centering
  \includegraphics[width=\textwidth]{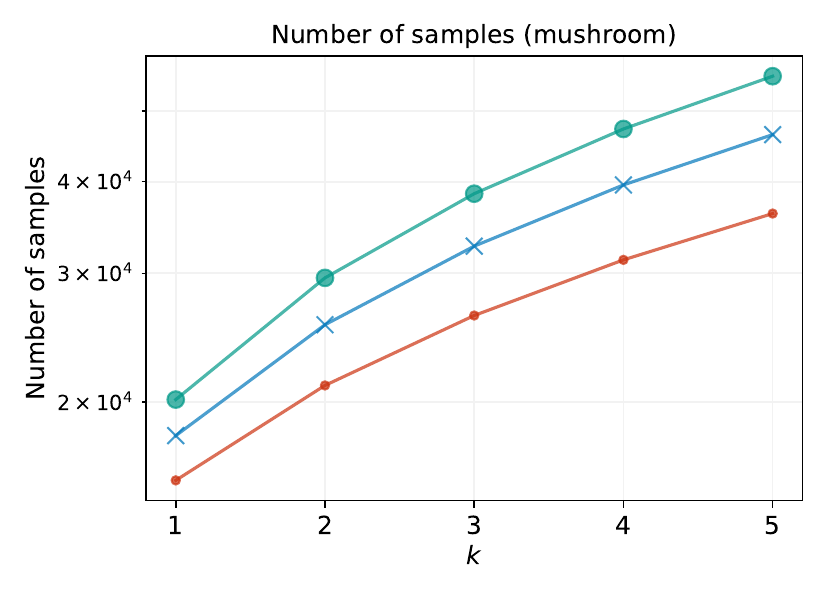}
\end{subfigure}
\begin{subfigure}{.32\textwidth}
  \centering
  \includegraphics[width=\textwidth]{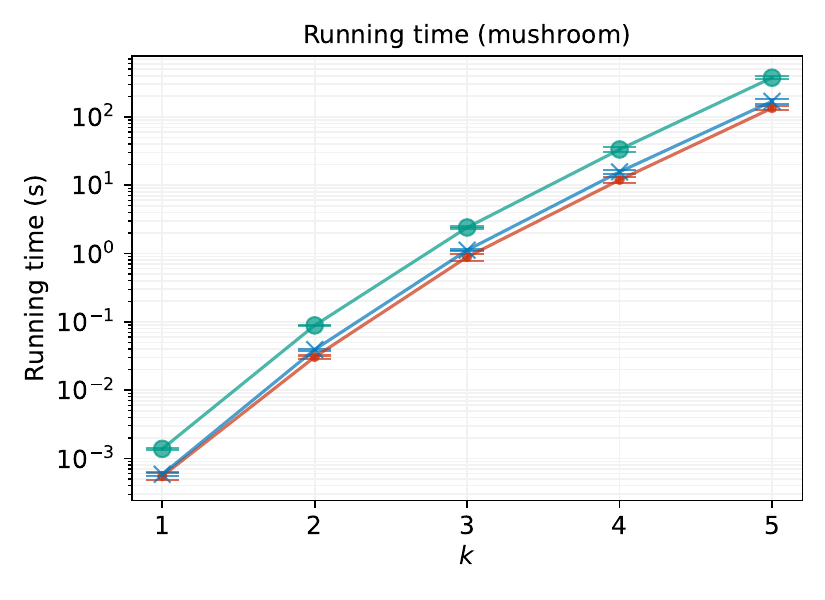}
\end{subfigure}
\begin{subfigure}{.32\textwidth}
  \centering
  \includegraphics[width=\textwidth]{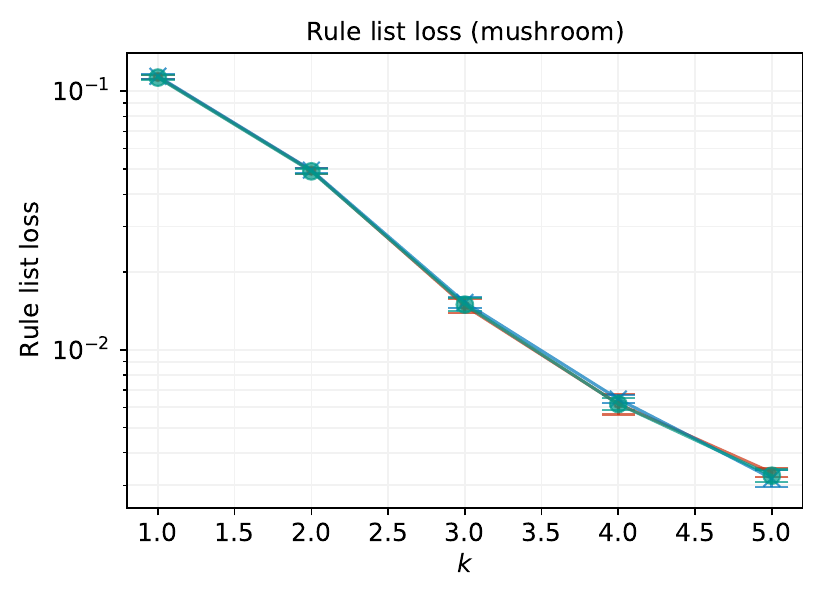}
\end{subfigure}
\begin{subfigure}{.32\textwidth}
  \centering
  \includegraphics[width=\textwidth]{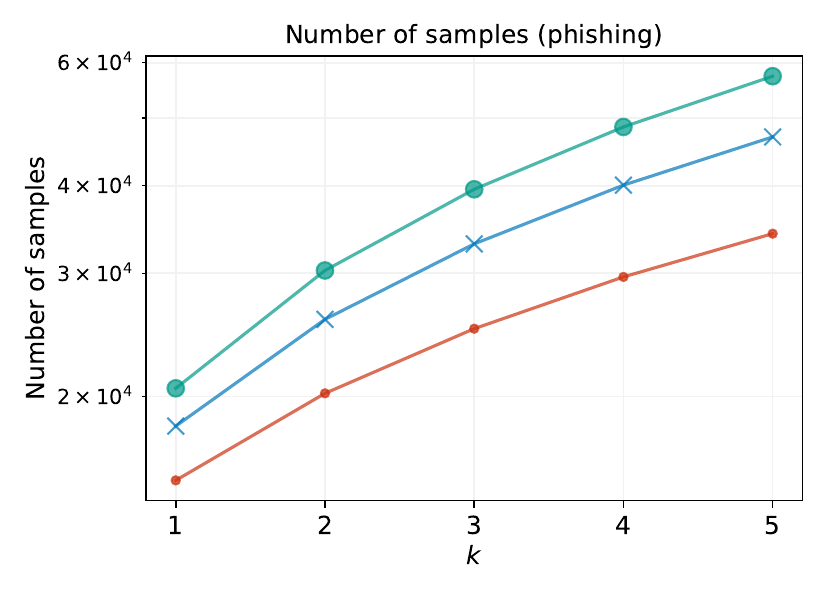}
\end{subfigure} 
\begin{subfigure}{.32\textwidth}
  \centering
  \includegraphics[width=\textwidth]{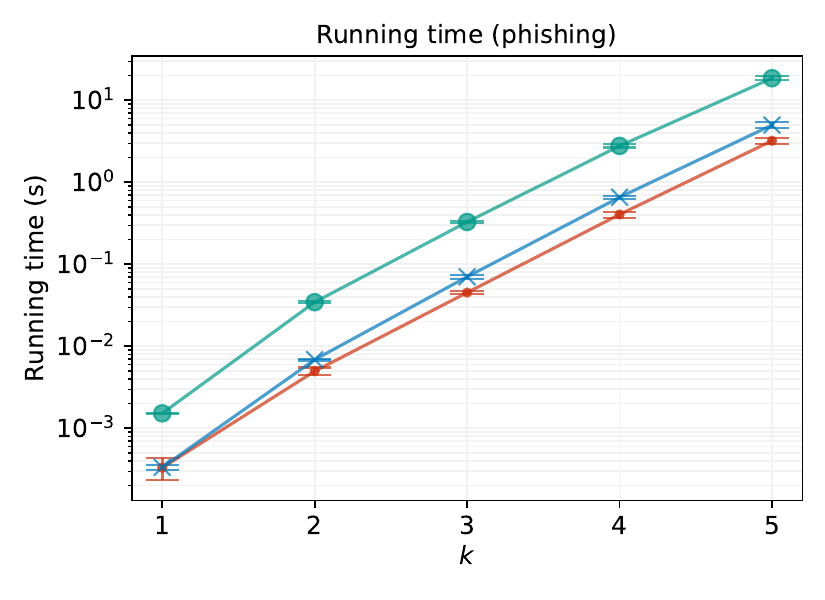}
\end{subfigure}
\begin{subfigure}{.32\textwidth}
  \centering
  \includegraphics[width=\textwidth]{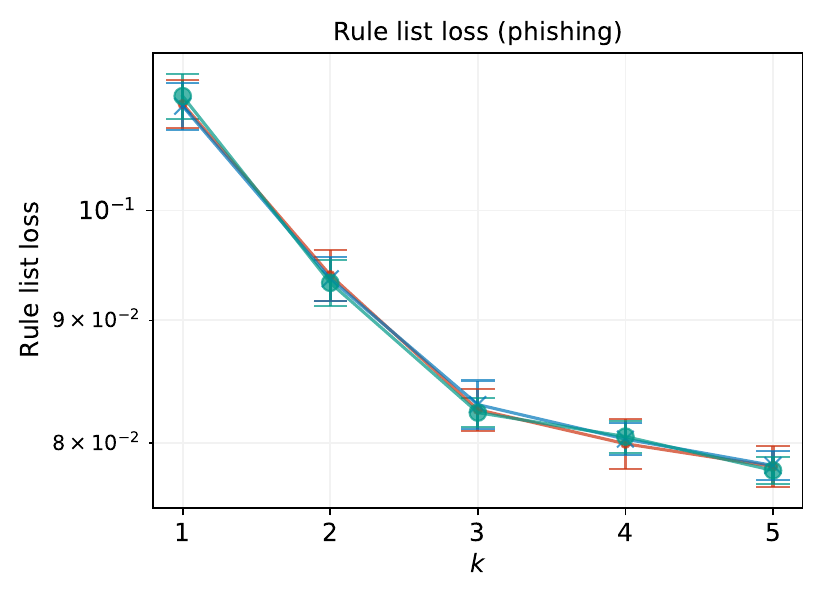}
\end{subfigure}
\caption{ 
Number of samples, running time, and rule list losses of \algname\ varying $z$ and $k$, using $\theta=0.025$, $\varepsilon=0.5$, on datasets mushroom and phishing. 
}
\label{fig:resultsparameters}
\Description{The figures show the impact to the running time and reported loss when varying the parameters z and k, that control the length and size of the rule lists.}
\end{figure*}
\fi

\end{document}